\documentclass[preprint,12pt]{elsarticle}

\usepackage[table]{xcolor}
\usepackage{amsmath}
\usepackage{xspace}
\usepackage{algorithm}
\usepackage{algcompatible}
\usepackage{algpseudocode}
\usepackage{verbatim}
\usepackage{pgfplots}
\usepackage{booktabs}
\usepackage{isotope}
\usepackage{siunitx}
\usepackage{caption}
\usepackage{subcaption}
\usepackage[shortlabels]{enumitem}
\usepackage{tikz}
\usepackage{amssymb}
\usepackage{adjustbox}
\usepackage{url}
\usepackage{array, makecell}
\usepackage{changepage,threeparttable}
\usepackage{multirow}
\usepackage{pifont}
\AtBeginDocument{%
  \providecommand\BibTeX{{%
    \normalfont B\kern-0.5em{\scshape i\kern-0.25em b}\kern-0.8em\TeX}}}

\usepackage{todonotes}
\presetkeys%
    {todonotes}%
    {inline,backgroundcolor=yellow}{}

\usepackage{amsthm}

\usepackage{ourMacros}

\newif\ifnotes
\notestrue

\ifnotes
\newcommand{\francesco}[1]{\textcolor{blue}{{\bf [Francesco: }{\em #1}{\bf ]}}}
\newcommand{\dietmar}[1]{\textcolor{green!55!blue}{{\bf [Dietmar: }{\em #1}{\bf ]}}}
\newcommand{\tommaso}[1]{\textcolor{magenta}{{\bf [Tommaso: }{\em #1}{\bf ]}}}
\newcommand{\lucio}[1]{\textcolor{violet}{{\bf [Lucio: }{\em #1}{\bf ]}}}

\def\wanote#1{\todo[size=\scriptsize,backgroundcolor=white,linecolor=red,bordercolor=red]{\textit{WA}: #1}} 
\else
\newcommand{\francesco}[1]{}
\newcommand{\dietmar}[1]{}
\newcommand{\tommaso}[1]{}
\newcommand{\lucio}[1]{}
\def\wanote#1{\todo[size=\scriptsize,backgroundcolor=white,linecolor=red,bordercolor=red]{}}
\fi

\pgfplotsset{compat=1.17}

\newtheorem{mydef}{Definition}

\journal{Information Sciences}

\begin{document}
\title{Conversational Recommendation:\\ Theoretical Model and Complexity Analysis}

%
%




\author[ba]{Tommaso Di Noia}
\ead{tommaso.dinoia@poliba.it}
\address[ba]{Politecnico di Bari, Italy}

\author[tus]{Francesco Donini\corref{cor1}}
\ead{donini@unitus.it}
\address[tus]{Università degli Studi della Tuscia, Italy}

\author[kl]{Dietmar Jannach}
\ead{dietmar.jannach@aau.at}
\address[kl]{University of Klagenfurt, Austria}

\author[ba]{\mbox{Fedelucio Narducci}}
\ead{fedelucio.narducci@poliba.it}

\author[ba]{Claudio Pomo}
\ead{claudio.pomo@poliba.it}

\cortext[cor1]{Corresponding author.}


\newtheorem{theorem}{Theorem}
\newtheorem{corollary}{Corollary}
\newtheorem{example}{Example}

\begin{abstract}
Recommender systems help users find items of interest in situations of information overload in a personalized way, using needs and preferences of individual users. 
In conversational recommendation approaches, needs and preferences are acquired by the system in an interactive, multi-turn dialog which is usually driven by incrementally asking users about their preferences about item features or individual items. A central research goal in this context is efficiency, evaluated with respect to the number of required interactions until a satisfying item is found. 
Today, research on dialog efficiency is almost entirely empirical, aiming to demonstrate, for example, that one strategy for selecting questions to ask the user is better than another one in a given application. 
With this work, we complement empirical research with a theoretical, domain-independent model of conversational recommendation. This model, designed to cover a range of application scenarios, allows us to investigate the efficiency of conversational approaches in a formal way, in particular with respect to the computational complexity of devising optimal interaction strategies. 
An experimental evaluation empirically confirms our findings.
\end{abstract}


\maketitle

\section{Introduction}
System-generated recommendations have become a common feature of modern online services such as e-commerce sites, media streaming platforms, and social networks. In many cases, the suggestions made by the underlying recommender systems are personalized according to the stated or assumed needs and preferences by the user. In the most prominent applications of recommender systems, \eg on Amazon.com, Netflix, or YouTube, the user preferences are estimated based on the observed past user behavior.  There are, however, also a number of application domains where no past interactions logs are available or where the user's needs and preferences might be different each time the user interacts with the service. Consider, for example, someone seeking a recommendation for a restaurant this evening for a party of four, where the requirements include that the location is nearby, that the prices are modest, and that there is a vegetarian option. In such a situation, the user's current needs and preferences have to be interactively acquired by the system to make a suitable recommendation.

The class of systems that support such interactions are called \emph{Conversational Recommender Systems} (\crs).
In these systems, the recommendation process consists of an interactive, multi-turn dialog, where the system's goal is to learn about the user preferences to the extent that appropriate recommendations can be made. The corresponding preference elicitation process can be implemented in different ways, ranging from predefined fill-out forms to natural language interfaces
---see \cite{jannach2020survey} for an overview. In particular for this latter class of interfaces we observed substantial progress in terms of voice recognition and natural language understanding in recent years, leading to the development of voice-controlled devices like Apple's Siri or Amazon's Alexa, and continuously improved chat-bot systems.

In that context, a typical goal when designing a \crs is to minimize the effort for users by asking as few questions as possible, \ie to increase the \emph{efficiency} of the dialog.


Today, research in this area is almost entirely empirical. In fact, typical research designs are based on simulations or user studies, in which two or more interaction strategies are compared in one or two application domains based on real or synthetic datasets. The corresponding efficiency measures are, for example, the number of required user interactions or the perceived difficulty and effort of the recommendation dialogs, where the main assumption is that a lower number of required interactions leads to a better usable system.

Such empirical studies are certainly important and insightful. However, little is  known about the theoretical aspects of the underlying interactive recommendation processes.
Unfortunately, theoretical questions regarding, \eg the computational complexity of determining a good or the best interaction strategy can not be answered without a formal characterization of the overall problem.

With this work, we address this research gap and provide a theoretical model of conversational recommendation. The model is designed in a domain-independent way and aims to cover a wide range of realistic application scenarios. A conversational recommendation process is modeled as a sequence of states, where state transitions correspond to common \emph{conversational moves} \cite{Reichman85} that can be found in the literature. Among the possible actions taken by the user we admit utterances for expressing preferences on items or features, or for relaxing or revise previously stated preferences.

Since our model is agnostic about the application domain and the algorithm that is used to select and rank the objects for recommendation---\ie the recommendation algorithm---, it serves as a basis to analyze important theoretical properties of conversational recommendation processes.

The main contribution of this work is the study of the computational complexity for finding an efficient conversational strategy in term of number of dialog turns.

From our study we found that:
\begin{itemize}
    \item the problem of finding an efficient conversational strategy in terms of number of dialog turns is NP-hard, but in PSPACE\footnote{PSPACE is a complexity class that contains all decision problems that can be solved by a deterministic Turing machine 
    by an algorithm whose space complexity is bounded by a polynomial function in the size of the input.};
    \item some specific factors of the item catalog influence the complexity of the problem;
    \item for a special class of catalogs, the upper bound lowers to POLYLOGSPACE\footnote{POLYLOGSPACE is a complexity class that contains all decision problems that can be solved on a deterministic Turing machine by an algorithm whose space complexity is bounded by a polylogarithmic function in the size of the input.}.
\end{itemize}
From a practical perspective, our analysis leads to the observation that the efficiency of a conversation strategy is tied to the characteristics of the available item catalog.
Observations from an empirical analysis on datasets based on MovieLens-1M support these theoretical considerations.

The paper is organized as follows. In the next section we discuss existing works and outline our research goals. We provide a model formalization in Section \ref{sec:formalization}, and then introduce our theoretical results in Section \ref{sec:devising-optimal}. Section \ref{sec:experiment} is devoted to an experimental evaluation to support the outcomes of our empirical analyses. A summary and outlook of the presented results, together with the identification of future directions close the paper.
In order to improve the paper readability, Table \ref{tab:symbols} and Table \ref{tab:acr} collect acronyms and symbols used in the paper.

\begin{table}[!htp]\centering
\caption{Table of symbols}\label{tab:symbols}
\rowcolors{2}{}{gray!15}
\scriptsize
\begin{tabular}{lrr}\toprule
\textbf{Symbol} &\textbf{Description} \\\midrule
$\sigma$ &Substitution from variable to value \\
$\hat{I}$ &Ideal item \\
\textit{AV} &Set of active values \\
\textit{B} &Boolean Domain \\
$\C$ &Catalog of Items \\
\textit{D} &Domain \\
\textit{f} &Feature \\
\I &Set of ideal items \\
\textit{I} &Ideal items \\
IS1 &\makecell[r]{Item set with only a few features, but with a\\larger number of distinct values} \\
IS2 &\makecell[r]{Item set with a larger number of features,\\but each of them only has a few distinct values} \\
$\K$ &Set of costraints \\
M &Maximum number of interactions \\
$\N$ &Disliked items \\
$\P$ &Liked Items \\
P1 &\makecell[r]{Protocol in which the CRS\\does not ask for disliked values when an item is rejected} \\
P2 &\makecell[r]{Protocol in which the CRS always\\asks for a disliked value when the user rejects an item,\\ the user provides such a value, and the CRS\\discards all items sharing the same feature value }\\
PRI &Positevely-rated Items \\
\textit{Q} &Query vector \\
\textit{R} &Recommender System \\
$\S$ &Set of recommandable items \\
\textit{T} &Set of terms (union of variables and values) \\
\textit{U} &User Model \\
\textit{V} &Union of all domains \\
\textit{X} &Set of variables \\
\bottomrule
\end{tabular}
\end{table}
\begin{table}[!htp]\centering
\caption{Table of acronyms}\label{tab:acr}
\rowcolors{2}{}{gray!15}
\scriptsize
\begin{tabular}{lrr}\toprule
\textbf{Acronyms} &\textbf{Description} \\\midrule
CRS &Conversational Recommender Systems \\
BTD &Binary Decision Tree \\
DT &Decision Tree \\
PRI &Positevely-rated Items \\
\bottomrule
\end{tabular}
\end{table}

\section{Previous Work and Research Goals} \label{sec:rel_work}
Dialog efficiency is one of the main dimensions in which \crs are evaluated in research literature \cite{jannach2020survey}. Here, efficiency usually refers to the time or effort that is needed by a user to find a suitable item. The corresponding underlying assumption is that users will find a system more useful if it requires less effort for them. Likewise, if a recommendation dialog takes too many steps, 
users might quit the conversation or, even worse, abandon the system as a whole.

\subsection{Increasing Dialog Efficiency}

Efficiency can in general be increased by designing the system in a way that guides users more directly to the items that match their preferences. In a non-interactive recommender system, efficiency is achieved through accurate predictions of what is relevant for users given their long-term preferences. In \crs, where the preferences are interactively elicited, one can \emph{in addition} try to find better conversation \emph{strategies}. By ``conversation strategy'' we mean the choices that \crs can make during the dialog. 
After a few interaction steps, a system might, for example, either decide to make a recommendation or to ask more questions in order to obtain a more complete picture of user preferences. While asking more questions means more interaction cycles in the short term, having more information might reduce the risk of providing irrelevant recommendations and more user effort in later phases of the dialog.
Another decision point exists after the system has decided to ask more questions. In this case, the system can usually select between several options regarding what to ask next.

In the literature, different technical approaches were proposed to increase the efficiency of the recommendation process though improved conversation strategies. Here, we give a few examples of different types of such approaches. Note, however, that with this work our aim is not to propose a new strategy, but to understand, on a theoretical level, how complex it is to make optimal choices. A summary of our analysis is reported in Table \ref{tab:dialog_management}.

The conceptually most simple approach is to build a \crs that is based on a static set of application-specific rules about how to continue the dialog. In the interactive sales advisory system Advisor Suite~\cite{Jannach:2004:ASK:3000001.3000153,jannach2007rapid}, for example, the possible dialog paths are predefined in terms of a state-machine. How the system can proceed in the dialog is thus defined through the transition graph. The actual choice of the system is then determined based on manually-defined decision rules. For example, the system could adapt the way the questions are asked depending on the user's self-reported expertise in the domain or based on previous user answers. To what extent such hand-crafted decision rules actually increase the efficiency of the conversations was unfortunately not evaluated in the mentioned papers.

Various more elaborate approaches were proposed in the context of \crs that are based on \emph{critiquing}. In critiquing-based systems \cite{chenpucritiquing2012}, the idea is that the \crs provides a recommendation relatively early in the dialog or even starts with an initial recommendation. Users can then apply critiques on a given recommendation---\eg in the form ``cheaper'' or ``lighter'' for a camera recommender system---after which the system makes a new recommendation.
The system therefore interactively refines the preference model until a recommendation is accepted (or no recommendation is found to be suitable). While such an approach in its basic form is intuitive and conceptually simple, it might turn out to be not very efficient in particular when there are many item attributes the user can apply critiques on. Depending on the implementation, a critiquing system might offer all attributes for critiquing at once or incrementally ask for desired feature values, also known as ``slot filling''~\cite{DBLP:conf/eais/BelliniBNSNP20, DBLP:journals/umuai/NarducciBGLS20}. In case of incremental slot filling, a naive implementation might end up asking too many questions about features that are either \emph{(i)}  not relevant for most of the users or \emph{(ii)} not particularly well-suited to narrow down the range of remaining options. 
\begin{adjustwidth}{-2.5 cm}{-2.5 cm}
\centering
\begin{threeparttable}[!htb]\centering
\caption{Dialog Management Overview}\label{tab:dialog_management}
\rowcolors{2}{}{gray!15}
\scriptsize

\begin{tabular}{wl{1.8cm}wc{1.8cm}wc{1.2cm}wc{3.5cm}wc{1.8cm}wc{1.8cm}}\toprule
\multirow{2}{*}{\textbf{Work}} &\multirow{2}{*}{\shortstack{\textbf{Interaction}\\\textbf{Strategy}}}
&\multirow{2}{*}{\shortstack{\textbf{Critique}\\\textbf{Based}}} 
&\multirow{2}{*}{\textbf{Next question}} &\multicolumn{2}{c}{\textbf{Preference Elicitation}} \\\cmidrule{5-6}
& & & &\textbf{Feedback} &\textbf{Item Selection} \\\midrule
\cite{Jannach:2004:ASK:3000001.3000153, jannach2007rapid} &Rule Based & & &  & \\
\cite{DBLP:conf/eais/BelliniBNSNP20} &Rule Based & & &\makecell{Feature-Based} &  \\
\cite{DBLP:journals/umuai/NarducciBGLS20} &Rule Based & & &\makecell{Feature-Based\\Item-Based} & Single \\
\cite{mccarthy2004thinking, chen2007preference, DBLP:conf/recsys/ViappianiPF07, Pu:CHI2006} & &Dynamic & &Feature-Based &Single \\
\cite{Reilly:ACBR2004} & &Compound & &Feature-Based &Single \\
\cite{Shimazu2002, DBLP:conf/eee/MirzadehRB05, Chakraborty2019RecommendenceAF, Yin:2017:DID:3097983.3098148} & Slot Filling  & &\makecell{Entropy-Based}&Feature-Based & \\
\cite{Sun:2018:CRS:3209978.3210002,tsumita2019dialogue, mahmood2009improving, Christakopoulou:2018} & Slot Filling  & &\makecell{Reinforcement Learning} &Feature-Based & \\
\cite{DBLP:conf/chi/LoeppHZ14} & & &\makecell{Information Gain} &Item-Based &Pair \\
\cite{narducci2018improving} & Slot Filling & & &\makecell{Feature-Based\\Item-Based} & \\
\cite{DBLP:conf/iui/CareniniSP03, AM:IUI2002, iovine2021empirical} & & &\makecell{Random/Popularity/\\Entropy/PopEnt/\\Item-Item personalized} &Item-Based &Single \\
\cite{christakopoulou2016towards} & & &\makecell{Random/Greedy/\\Maximum Variance/\\Maximum Item Trait/\\Minimum Item Trait/\\Upper Confidence/\\Thompson Sampling} &Item-Based &Pair \\
\cite{Zhao:2013:ICF:2541154.2505690} & & &\makecell{Thompson Sampling \\with PMF} &Item-Based &Single \\
\bottomrule
\end{tabular}
\end{threeparttable}
\end{adjustwidth}
\vspace{10px}
To deal with such problems, \emph{dynamic} and \emph{compound} critiquing approaches were proposed  \cite{mccarthy2004thinking,Reilly:ACBR2004,chen2007preference,DBLP:conf/recsys/ViappianiPF07,Pu:CHI2006}. In the compound critiquing approach \cite{Reilly:ACBR2004}, for example, the idea is not only to present critiques to the user that concern more than one feature (\eg ``less expensive \emph{and} lighter''), but to determine possible compound critiques based on the properties of the item catalog. During an ongoing session, it might for example be most helpful to propose compound critiques that, if applied, would rule out a larger fraction of the available options.

The dynamic and data-based selection of the next question to ask or, more generally, the next conversational move, is however not limited to critiquing-based approaches. In the context of such slot-filling strategies---where the system questions mostly relate to item properties (or: facets)---various approaches were put forward to determine the best order of the questions. 
Often, such approaches are \emph{entropy-based} and consider the potential effects of individual questions (and their answers) on the remaining space of options
\cite{DBLP:conf/eee/MirzadehRB05,Shimazu2002,Chakraborty2019RecommendenceAF, Yin:2017:DID:3097983.3098148}. In some cases, feature popularity information is considered as well \cite{DBLP:conf/eee/MirzadehRB05}.
Several alternative technical approaches to decide on the next action are based on machine learning, using, for example, reinforcement learning techniques \cite{Sun:2018:CRS:3209978.3210002,tsumita2019dialogue,mahmood2009improving} or recurrent neural networks \cite{Christakopoulou:2018}.
\begin{table}[!htp]\centering
\caption{Dialog Efficiency Overview}\label{tab:metrics}
\rowcolors{2}{}{gray!15}
\scriptsize
\begin{tabular}{lcc}\toprule
\textbf{Work} &\textbf{Metric} &\textbf{Interaction Mode} \\\midrule
\cite{Sun:2018:CRS:3209978.3210002} &Number of Interactions &Simulation/Real \\
\cite{adomavicius_new_2007, DBLP:conf/recsys/ViappianiB09,Grasch:2013:RTC:2507157.2507161} &Number of Interactions &Simulation \\
\cite{DBLP:conf/ecis/ChenHC08, ResearchNoteContingencyApproach2013,Pecune:HAI2019,DBLP:conf/aaai/ChenP06,MapBasedRicci2010,DBLP:conf/chi/LoeppHZ14} & Questionnaires &Real \\
\cite{McCarthy2004Onthe,mccarthy2006group, Jin:2019:MEC:3357384.3357923} &Questionnaires &Simulation \\
\cite{warnestaal2005user,Ikemoto2019} &Completion Time &Real \\
\cite{DBLP:conf/ah/ZhangP06,reilly2007comparison,mahmood2009improving} &Completion Time/Questionnaires &Real \\
\cite{iovine2020conversational} &Completion Time/Questionnaires &Simulation \\
\bottomrule
\end{tabular}
\end{table}
Instead of asking for desired item features, some approaches are based on eliciting feedback on items considered as a whole. This could be done by asking for like or dislike statements for individual items or by asking for the relative preference regarding two or more items \cite{DBLP:conf/chi/LoeppHZ14}. One main question in that context is how to select these items or sets of items to maximize the efficiency of the process. Such a selection can for example be based on item popularity or based on diversity considerations \cite{narducci2018improving,DBLP:conf/iui/CareniniSP03,AM:IUI2002}.
In some approaches, the system can also learn to decide if an absolute rating should be elicited or pairs of items should be presented to the user \cite{christakopoulou2016towards,Zhao:2013:ICF:2541154.2505690, iovine2021empirical}.

\subsection{Measuring Dialog Efficiency}
The most common approach to measure dialog efficiency in the literature is to count the number of required interaction cycles until a recommendation is accepted by the user. Such an approach is followed in many works that are based on critiquing, but also in chatbot-like applications and in more recent learning based systems
\cite{McCarthy2004Onthe,mccarthy2006group,adomavicius_new_2007,DBLP:conf/recsys/ViappianiB09,mahmood2009improving,
DBLP:conf/ah/ZhangP06,reilly2007comparison,MapBasedRicci2010,Grasch:2013:RTC:2507157.2507161,Ikemoto2019,
narducci2018improving,Jin:2019:MEC:3357384.3357923,warnestaal2005user,Sun:2018:CRS:3209978.3210002}. In most of these works, the interaction between a user and the \crs is \emph{simulated}. In such simulations a software agent is developed, that has a certain preference profile and answers requests made by the system based on the underlying preferences. The common assumptions here are that the agent knows all preferences from the beginning, behaves rationally and truthfully, and does not change its mind during the conversation. In reality, not all assumptions might hold, and we include this aspect in our simulations (see later sections), by having users  revise their preferences, \eg in case no item matches the initial preferences.

Alternative efficiency measures that are often used in the context of user studies include the \emph{task completion time}  \cite{mahmood2009improving,Jin:2019:MEC:3357384.3357923,iovine2020conversational}. We may in general assume that shorter task completion times usually means that the recommendation process was more efficient. However, a longer task completion time can, depending on the domain, also mean that the users found more interesting options to explore \cite{DBLP:conf/ecis/ChenHC08}.

In addition to these objective measures---number of interaction turns and task completion time---some researchers also rely on \emph{subjective} measures in the context of studies that involve humans. In such studies, the participants are usually asked after the main experiment task how they perceived the effort that was needed, \eg to find a suitable item. In most cases, efficiency is only one of several variables that are assessed in such post-task questionnaires. Examples of works that use efficiency-related measures are \cite{ResearchNoteContingencyApproach2013,Pecune:HAI2019,DBLP:conf/chi/LoeppHZ14,DBLP:conf/aaai/ChenP06,mahmood2009improving}.
Table \ref{tab:metrics} summarizes the metrics adopted in each analyzed work.

\subsection{Research Goals}
\paragraph{Research Gaps and Goals}
Existing research on the efficiency of different conversational recommendation strategies, as mentioned above, is entirely empirical. Either the research is based on studies with simulated users, or it is based on user studies in which participants usually interact with a prototype system. In either case, such studies are based on one or a few particular experimental configurations. In particular, often only one specific item catalog is considered when the efficiency of a newly proposed conversation strategy is demonstrated. However, as the theoretical analysis shows later in this paper, the characteristics of the catalog, in particular, in terms of the number of item features, and the number of feature values that are shared by several items, may have an impact on the efficiency of a given strategy. Depending on such characteristics, for example, in one case it might be much more efficient to ask the user about a feature preference than asking for an item preference, whereas in another situation, the advantage is negligible.

Moreover, existing research has not looked yet at the computational complexity of choosing an optimal strategy in a conversation. Our research aims at closing this gap and provides a formal definition of the conversational recommendation problem, based on which such complexity analyses are performed. Furthermore, we complement our theoretical analysis with results obtained from simulation-based experiments with data elicited from a real dataset (MovieLens).

\paragraph{Research Scope}
A multitude of \crs have been proposed in the past, focusing on a variety of different problem domains. Depending on the target domain, such \crs are often diverse in terms of what functionality they offer. Specifically, existing \crs vary largely in terms of the supported user intents. In their recent work, Cai and Chen~\cite{CaiLi2020Umap} developed a comprehensive list of user intents and corresponding system actions (or: conversational moves) in \crs.

In our work, we focus on the most common \crs in the literature: those that base their conversations on desired or undesired values for item features (``slot filling'') and/or the acquisition of preferences of users regarding an item as a whole. Our formalization therefore includes actions that, for example, relate to the initial acquisition of preferred feature values, but also considers additional user actions in such applications, such as the relaxation or revision of constraints \cite{jannach2008finding}. The resulting formalization therefore represents an abstraction from various existing approaches from the literature, and the resulting insights therefore apply to a larger range of \crs implementations. In this context, note that our formalization considers the actual item ranking task as a black box. This means that our theoretical framework is independent of the specific type of algorithm that is used for determining the order of the recommendations. Thus any type of approach, \eg collaborative, content-based, or hybrid, can be used internally.

Regarding the measure for efficiency, our discussion showed that while a number of alternative measures were applied in the literature, counting the number of required interaction steps is the predominant choice of researchers. Therefore, we also use this measure in our theoretical analysis and in our simulations.

\section{Model Formalization} \label{sec:formalization}
On a conceptual level, a \crs---as considered in our study---works as follows. The system main task in the interaction with the users is to elicit their preferences\footnote{In this paper, we use the words ``preference'', ``desire'', ``needs'' as synonyms.} regarding item features and items. To that purpose, the system maintains information about these preference statements in a user model. These preference statements determine which of the items of a given catalog qualify to be recommended. Therefore, in our theoretical framework we assume a retrieval-based item filtering approach, which is commonly used in critiquing-based and constraint-based approaches to recommendation \cite{felfernig2015constraint}. In analogy to database-oriented approaches, we therefore use the term ``query'' when referring to the user's preferences, as these preferences lead to the retrieval of matching items from the catalog. The retrieved items are then ranked according to their assumed utility for the user. Such a ranking can be based on any type of information, \eg the popularity of certain items. In order to carry a general analysis, in our approach we abstract from the details of this ranking, as mentioned above.

To model the conversational recommendation process, we rely on the notion of \emph{state} of a conversation, and what transformations this state can be subject to, depending on the interaction. Each expressed user preference, for example, leads to a change of the state of the conversation and may also imply that the set of suitable item recommendation changes. This formalization through conversation states ultimately serves us as a basis to study the efficiency of conversational strategies, in the sense that the most efficient conversational strategies will be the ones minimizing the number of states the conversation must pass through, to reach an end.

\subsection{Basic components}
We summarize below the components and symbols of our formalization, and we explain them in detail in the rest of the section.
To give to the reader an immediate grasp of our formalization, we summarize in the schema below the function of each component of the user model:
\vskip 1\baselineskip
\noindent
\begin{tabular}{c||c|c}
                      & positive preferences\ldots & negative preferences\ldots \\ \hline\hline
 \ldots about features & Query $Q$  & Constraints $\K$ \\ \hline
  \ldots about whole items & set of items \P  & set of items \N \\ \hline
\end{tabular}
\vskip 1\baselineskip
\noindent
More formally, we have:
\begin{enumerate}
    \item a \emph{Catalog} of items \C;
    \item the user model, which is a 4-vector $U=\Vector{Q,\K,\P,\N}$, where the \emph{query} $Q$ denotes the user's positive preferences regarding item features (\emph{desiderata}), \K are constraints, in the form of undesired values for a feature, \P is a set of positively-rated \emph{preferred} items, and \N is the set of items that are negatively rated by the user;
    \item the \emph{Recommender} system \Rec (a black box in this paper), that given $U$ ranks items in \C, and returns the top-k items, which we abbreviate as $\Rec(U)$; we stress our being agnostic in modeling \Rec, which could be content-based, or based on collaborative filtering, etc. Following the intuition, we assume that $\Rec(U) \cap \N = \emptyset $---that is, no negatively-rated item is ever proposed;
    \item the \emph{state} of the conversation as the 5-vector $\Vector{Q,\K,\P,\N,\Rec(U)}$ which we denote\footnote{Inspired by Python, we use ``+'' for vector concatenation, \eg $\Vector{A,B,C}+\Vector{D,E}=\Vector{A,B,C,D,E}$. } also as $U+\Rec(U)$, \ie the state is composed by all of the above elements but for the catalog, which we consider constant during the conversation;
    \item the interactions between the user and the recommender as transformations  of  $U$, denoted by $\tau$.
\end{enumerate}
This formalization provides us with the necessary notation to study \emph{strategies}, \ie algorithms that, given a state of the conversation, suggest the next move in order to proceed towards a successful recommendation.

To ease reading, we adopt the following syntax conventions: atomic elements are denoted by lowercase letters, \eg $x,t,v$; sets and vectors of atomic elements are denoted by uppercase letters, \eg $C,D,I,Q$; and sets and vectors of the previous sets and vectors are denoted by calligraphic letters, \eg $\C, \K$. We now delve into each element in detail.

\paragraph{Features}
Let $\{f_1, f_2,\ldots, f_p\}$ be a set of \emph{features} for items and queries, where each feature $f_i$ can take values in a \emph{domain} $D_i$. In the rest of this paper, we assume that features are only single-valued, \ie we do not consider set-valued features. When making examples, we often characterize the feature with a name instead of the index, as in $f_{cuisine}, f_{director}$.
We highlight the  \emph{Boolean domain} $B = \{\false,\true\}$ as a special one.
While any domain $D$ could be translated into a number of $|D|$ Boolean domains via one-hot encoding, we prefer not to use such an encoding because of its exponential-space blowup\footnote{The values of a feature need $\lceil\log_2(|D|)\rceil$ bits to be encoded, while their one-hot encoding needs $|D|$ bits.} that would hamper our subsequent computational study. We denote the union of all domains as $V = \left(\bigcup_{i=1}^p D_i \right)$.

\paragraph{Variables}
In addition to domain values, let $X = \{x_1,x_2,\ldots\}$ be a set of variables. Intuitively, a variable is a placeholder for a query feature which the user has not yet expressed a preference on.
As an abstraction of both variables and values, we define terms. A \emph{term} $t$ can be either a value in a domain, or a variable; so the set of terms is  $T = V \cup X$. Terms are needed to model user queries, which we introduce next.

\paragraph{Queries}
The user's stated preferences are represented by a \emph{query vector} of terms $Q = \Vector{t_1, t_2,\ldots, t_p}$, where each position $i$ (called \emph{slot}) in the vector is associated with a feature $f_i$; features act as names for vector positions---similar to field names in relational databases.
We say that ``the feature $f_i$ contains the term $t_i$ in $Q$'' when the $i$-th position in $Q$ holds the term $t_i$---written $f_i(Q) = t_i$.
 Intuitively, when a feature $f$ contains a variable $x$, it means that we have no knowledge about the user's preferences in this respect. Remember that one goal of the interactive process is to fill in such features/slots.

\paragraph{Constraints}
A query represents the \emph{positive} part of a user's  desires;
to represent the negative part---what the user definitely does not want---we define the \emph{constraints} on $Q$  as another vector $\K = \Vector{C_1,C_2,\ldots,C_p}$, where each constraint $C_i \subset D_i$ is a set of values the user dislikes for feature $f_i$. Such a constraint represents the negative part of the history of the conversation, namely, the negative answers given by the user during the interaction---\eg ``I don't want horror movies tonight'', which is formalized as $\mathit{horror} \in C_{genre}$. Observe that for such negative constraints, we limit this formalization to the most common case implemented in CRS, that is, we consider here only single values.

\paragraph{Items}
We represent items using the same feature space as the queries, but for variables.
Given  a catalog  $\mathcal{C} = \{I_1, I_2,\ldots \}$ of items, we let an item $I= \Vector{v_1, v_2, ..., v_p}$ be a vector where each feature $f_i$ in position $i$ contains value $v_i$---written $f_i(I)=v_i$. In this study we consider only \emph{complete knowledge}, \ie all feature values are known.\footnote{Extending the formalization to incomplete item knowledge is left for future work.}

\paragraph{Substitutions}
During the interactive recommendation process, we continuously learn more about the user's preferences, which leads to changes in the user model $U$. We call such changes \emph{transformations}.
We define the simplest form of transformation as slot filling, which in our setting means substituting a variable with a value.
A \emph{simple substitution} $\sigma: X \rightarrow V$ is a function mapping a variable $x$ to a value $v$. Apart from such a change, the substitution leaves the rest untouched---more formally, the substitution $\sigma$ is the identity function on every term but $x$, which is changed to $v$. Recall that we are interested in substitutions in the user model $U$, and not in the items, whose descriptions do not contain variables.

We extend $\sigma$ to vectors position-wise, \ie we have $\sigma(\Vector{t_1,t_2,\ldots,t_p}) = \Vector{\sigma(t_1),\sigma(t_2),\ldots,\sigma(t_p)}$. In this way, we can use the notation $\sigma(Q)$, knowing that $\sigma$ will change only the feature containing $x$.
Clearly, if a user already stated that s/he dislikes a specific feature value, such a value should never be proposed again during the conversation. Formalizing such an intuition, we say that a substitution is \emph{coherent with $\K$} if it does not substitute a variable with a value forbidden by the constraints, in formulas: for all $i=1,\ldots,p$ it holds that $\sigma(t_i) \not\in C_i$. From now on, we consider only substitutions coherent with $\K$.

We let substitutions $\sigma_1,\sigma_2,\ldots,\sigma_h$ be repeatedly applied to a term as $\sigma(t) = \sigma_1(\sigma_2(\ldots(\sigma_h(t)\ldots))$, and also such nested substitutions $\sigma$ can be applied to vectors position-wise, written as $\sigma(Q)$\footnote{Even if each substitution of the sequence changes one variable with a value, such a compositional extension allows us to write simply $\sigma(Q)$ instead of specifying which simple substitution changes which feature, as it would be in heavier notations like $\Vector{\sigma_1(t_1),\sigma_2(t_2),\ldots,\sigma_p(t_p)}$.}.
Moreover, the sequence of such nested substitutions preserves  the order (in reverse)
in which the subsequent slot-filling interactions change the state of the conversation.

 An item $I$ \emph{matches} a query $Q$ if there is a substitution $\sigma$ coherent with $\K$, such that $\sigma(Q) = I$. In practice, $\sigma$ formalizes all the slot fillings that would lead from $Q$ to the proposal of (exactly) $I$.
We say that a substitution $\sigma$, coherent with constraints $\K$, \emph{satisfies a query $Q$ in \C}
if there is an item $I \in \C$ such that $\sigma(Q)= I$.
Given a catalog \C, we write $Q(\C)$ for the set of items matching $Q$, in formulas: $Q(\C) = \{ I \in \C ~|~ \exists\sigma : \sigma(Q) = I \}$.


We note that when a feature $f$ contains a variable---\ie when the user specifies no preference for $f$---we do not know whether the user \emph{does not care} about $f$---so that the user would accept any value for $f$---or the user \emph{does not know} that there are some alternatives that fit her preferences while others do not. For instance, suppose a restaurant recommender describes items with some features, among which there is a Boolean feature regarding dress code; and suppose that when starting the conversation, the user did not mention preferences about dress code. When facing the proposal of a restaurant whose dress code feature is \true, the user might say ``\textit{no, I don't want dress code restaurants; sorry I didn't mention it before}''. Formally, this would mean that the conversation starts with a query $Q_0$ in which $f_{dressCode}(Q_0)=x$, the recommender (after some interaction) proposes a restaurant $I$ for which $f_{dressCode}(I)=\true$ (so $I$ matches $Q_0$ with $\sigma(x) = \true$), the user rejects it with the above dialog, and the conversation proceeds with the updated constraints $\K'$ in which $C_{dressCode} = \{\true\}$ (the disliked value).

\paragraph{Positively- and negatively-rated items.}
Both $Q$ and \K formalize the ``analytic'' part of a user's preferences, \ie what s/he thinks about particular values of item features. However, sometimes the user may want to express a holistic 
opinion on an item as a whole, \eg ``\textit{I don't want to go to restaurant Y anymore, I had a bad impression last time I went}'' (dislike), or ``\textit{I liked movie Z very much}'' (preference). We model such statements as two sets of items, $\P \subset \C$ (positive rating) and $\N \subset \C$ (negative rating).
\P and \N are part of the user model $U$ as mentioned above, so that our formalization is as general as possible.
We exclude from the formalization of the conversation the trivial solution that directly proposes items in \P; such items are not proposed (\eg in the movie domain, they might be movies already seen; in the restaurants domain, they could be appreciated restaurants far away from the user's current position).

\subsection{States and Transformations}
A \emph{state} of a \crs is the 5-vector $U+\Rec(U)$. The conversation starts with $U_0=\Vector{Q_0,\K_0,\P_0,\N_0}$ as the initial information about the user, and $U_0+\Rec(U_0)$ as the initial state.
If there are no past item ratings available from previous interactions, we face a `cold start' situation where  $\P_0 = \N_0 = \emptyset$. When a new conversational session starts, we let
$Q_0 = \Vector{x_1, x_2, \ldots, x_p}$ (all variables) and $\K_0 = \Vector{\emptyset,\emptyset,\ldots,\emptyset}$ (all empty sets), \ie no
preference on individual features was expressed yet.
Recall that we assume the catalog \C as constant in time---\ie no items enter or exit the catalog during the conversation.

In our model, we are formalizing system-driven conversations, where a conversation consists of a sequence of interactions. Given a state $U+\Rec(U)$, the interactions initiated by the system can be one of the following.
\begin{enumerate}
    \item\label{slot-filling-interaction} the system asks the user to \emph{fill in} (provide) a value for a particular feature under-specified so far (a variable in $Q$);
    \item\label{contraction-interaction} the system asks the user to enlarge a too narrow choice for a feature value in $Q$;
    \item\label{revision-interaction} the systems asks for changing a feature value in $Q$;
    \item\label{ask-disliked-value} upon rejection of one or more items, the system can ask the user whether there is a specific feature of these item(s) she dislikes.
\end{enumerate}
The user can react to the above system prompts with one of the following interactions:
\begin{enumerate}[(a)]
    \item\label{reject-item} given one or more recommendations $\Rec(U)$, the user can accept one of them, or reject them;
    \item\label{dislike-value} the user can state that she dislikes every item whose feature $f$ is filled with a particular value $v$ (\eg ``\textit{I don't like green cellphones}'')
\end{enumerate}

While this set of possible interactions might seem limited, remember that it covers a wide range of conversation strategies that were described in the literature, as discussed in Section~\ref{sec:rel_work}.
Regarding the specific user reactions, Interaction~\ref{reject-item} amounts to adding the rejected item(s) to \N, while Interaction~\ref{dislike-value} amounts to adding the value $v$ to a set $C_i$ in the constraints vector \K.
Regarding system prompts, in what follows we show how such interactions can be expressed in terms of \emph{transformations} on $U$, which extend substitutions as follows: a transformation $\tau: T \rightarrow T$ can map any term to another term. In this way, substitutions are just a particular case of transformations. We need transformations to model system-user interactions in a more general way than just turning a variable into a value (as is the case for slot filling).

\subsubsection{Interaction~\ref{slot-filling-interaction}: Slot filling} Suppose the user issued a very general query $Q$ to start with, \eg ``\textit{I would like to go to a restaurant tonight}''. The system has so many answers to this query that presenting only the top-$k$ ranked items $\Rec(U)$ might miss a satisfying experience, and given that $f_{cuisine}(Q) = x$, it asks in turn: ``\textit{Which kind of cuisine do you prefer?}'' and provides a set of possible feature values of the available items. When this set is too large to be presented to the user, another variant could be that the system actively suggests some value, \eg ``\textit{What about French cuisine?}'' given that there are plenty of French restaurants in the catalog. The formal counterpart of this dialog is a simple substitution $\sigma$ such that $\sigma(x) = \mathit{French}$, expressing the user's choice, and after this simple transformation is performed, the new state is $U' + \Rec(U')$, where $U' = \Vector{\sigma(Q),\K,\P,\N}$, and $f_{cuisine}(\sigma(Q)) = \mathit{French}$.

\subsubsection{Interaction~\ref{contraction-interaction}: Slot unfilling} Here  transformations more general than substitutions come into play. Imagine a conversation between a restaurant recommender and a user, in which after some questions (slot fillings) regarding cuisine, location, price level, etc.\ the recommender asks the user: ``\textit{What about Restaurant $I_5$? Japanese cuisine, midtown, upper-level price, no dress code?}''; yet getting  a ``Dislike'' from the user. Suppose there are no more restaurants with the features set so far. Then the recommender may say ``\textit{May I propose you other types of cuisine?}''. Formally, $\N' = \N \cup\{I_5\}$ (since Item~$I_5$ is rejected), and the query $Q$ in which $f_{cuisine}(Q) = \mathit{Japanese}$, is transformed by $\tau(Japanese) = x$ (a fresh variable), yielding a new state $U'+\Rec(U')$ with $U'=\Vector{\tau(Q),\K,\P,\N'}$, so that $f_{cuisine}(\tau(Q)) = x$.
This interaction is a first convenient way of driving a conversation out of a dead end, without committing on a new choice.\footnote{This kind of transformation has also been called \emph{query relaxation} \cite{Mirzadeh2004,jannach2008finding}.}
A second way follows below.

\subsubsection{Interaction~\ref{revision-interaction}: Slot Change} This interaction is similar to the previous case, but now the system proposes a specific change to the user: ``\textit{I have no other Japanese restaurants in midtown, upper-level, with no dress code; but downtown, there are a few. Are you willing to change midtown for downtown? Other alternatives are: north outskirts, or riverside}''\footnote{Algorithms for computing such alternatives are proposed, \eg in \cite{Jannach2004,DBLP:conf/ijcai/FelfernigFSMMT09}.} Of course, if the user accepts an alternative, the feature should change value. The formalization is also similar to the previous case: the conversation is in a state $U+\{I_5\}$, and feature $f_{location}(I_5)$ contains the value $midtown$. Again, since the user dislikes $I_5$, it is now $\N' = \N \cup\{I_5\}$. We let $\tau(midtown) = downtown$ (or whatever other value is chosen by the user), stepping into the new state $U' + \Rec(U')$ with  $U'=\Vector{\tau(Q),\K,\P,\N'}$ so that $\Rec(U') \neq \emptyset$ (since there are recommendable Japanese restaurants downtown).

We note that in the area of Knowledge Representation, slot unfilling could be considered a special case of knowledge \emph{contraction} \cite{colu-etal-ECRA05}, while slot change is a form of \emph{revision} \cite{gard-TARK96}. We occasionally use such names instead of  slot unfilling and slot change in the rest of the paper.

The formalization described so far can be distilled into a definition of conversation as follows.
\begin{mydef}
A \emph{Conversation} is a sequence of transformations $\tau_1,\tau_2,\ldots,\tau_z$ that  starting from an initial state $U_0 + \Rec(U_0)$ ends in a final state of success (the user chooses an item), or failure.
\end{mydef}
We observe that failure may occur because \emph{(i)} either a maximum number of interactions has been reached, or \emph{(ii)} the user gives up after a number of interactions, or \emph{(iii)} the conversation reaches a state $U + \emptyset$ in which it happens that both the \crs has no more items satisfying the user's preferences, and the user does not want to reconsider any of the  preferences stated in $U$.

\section{Devising Optimal Strategies} \label{sec:devising-optimal}
In the development of \crs with system-driven conversations, a cornerstone is the method  the \crs uses to decide how to proceed in the conversation; we call such a method a \emph{strategy}. Our goal in this section is to study the problem of developing efficient conversation strategies for \crs.
For our setting, a \crs \emph{strategy} is an algorithm that, in every possible state of a conversation, suggests the system its next action, among the ones formalized above.
As discussed in Section~\ref{sec:rel_work}, the efficiency of a \crs strategy can be studied in terms of the number of interactions the user needs to reach an item $I$. One can consider the worst case, \ie the maximum number of interactions (for some ``difficult'' item), or the average case, with respect to a probability estimating how likely item $I$ will be the chosen one. In this paper we study the minimization of the maximum number of interactions, leaving the probabilistic setting for future work.

To study conversation strategies, we need some reasonable assumptions about how the user answers the system prompts. In our work, we assume that there is a  set $\I$ of \emph{ideal} items the user would be satisfied with.
Without loss of generality, since we are studying efficiency, we assume that $\I \cap \C \neq \emptyset$, meaning that there always exist at least one item in the catalog that would satisfy the user, and the only point is how many interactions are needed to find it\footnote{This means that we do not consider here failure cases; in particular, we consider neither the case where there is a maximum number of interactions (Case~\emph{(i)} above), nor the case where the user gives up (Case~\emph{(ii)} above). Moreover, we consider no cutoff by the system (which would hide worst cases beyond that threshold).}.
Observe also that we do not set $\I \subset \C$, \ie we assume that there are ``ideal'' items the user would like, but which are not available in the catalog. They are exactly such items that make search difficult, since the user might drive the search (in the conversation) towards such unavailable items, heading towards available ones only later on.



We assume that the user, during the conversation, adopts (even unintentionally) the following \emph{truthful attitude} when interacting with the \crs:
\begin{itemize}
    \item the user chooses a slot filling $\sigma(x) = v$ for a feature $f_i$,  \emph{if and only if} there is at least one ideal item $I \in \I$ whose feature $f_i$ is actually filled by $v$---in symbols, $\exists I\in \I : f_i(I) = v$. Note that this is a double implication, \ie if there is no such ideal item, the user does not choose (or accepts as a suggestion)  to fill the slot $f_i$ with value $v$;

    \item similarly for slot change: the user chooses (or accepts the suggestion) the change of a slot value $\tau(v) = w$ for a feature $f_i$,  \emph{if and only if} there is at least one ideal item $I \in \I$ such that  $f_i(I) = w$;

    \item user's like/dislike statements about whole items are always coherent with ideal items \I, \ie when the user states that she likes an item $I$, it is $I \in \I$, and  when the user states that she dislikes an item $I$, then $I \not\in \I$.
\end{itemize}



\subsection{Lower Bounds}

We first show that our problem of devising efficient strategies for \crs includes Binary Decision Trees (BDT) minimization as a special case.
%
%
We recall the definition of a BDT \cite{Moret82} for our purposes. A BDTs is a procedural representation of a \emph{decision table} defined as a tabular representation of a function whose domain is contained in $B^p$ and whose range is the set of decisions. Given integers $p$, and $q \le 2^p$, a Decision Table is a $q \times p$ Boolean matrix\footnote{The use of the symbol $p$ (the same symbol used for the number of features in our formalization) is intentional here, as it will be clear later on.} in which the $p$ columns represent conditions/tests that can be \true or \false, and the $q$ rows represent the combinations of such conditions---one of which has to be assessed to take the appropriate decision. Each decision could be associated to one or more rows---all the combinations of conditions in which this decision has to be taken---and rows corresponding to impossible combinations are usually omitted. We give an example in Table~\ref{tab:example:DT:BDT} (left).
\begin{table}
\begin{tabular}{c|c}
    \begin{tabular}{cccc}
$x_1$   & $x_2$     & $x_3$ &  decision \\ \hline
\false  & \false    & \false    & $d_1$ \\
\false  & \false    & \true     & $d_1$ \\
\false  & \true     & \false    & $d_2$ \\
\false  & \true     & \true     & $d_3$ \\
\true   & \false    & \false    & $d_4$ \\
\true   & \true     & \false    & $d_4$ \\
\true   & \false     & \true    & $d_5$ \\
\true   & \true     & \true     & $d_5$
\end{tabular}
     &
\adjustbox{valign=m}{%
\begin{tikzpicture}
[
    level 1/.style={sibling distance=10em},
    level 2/.style={sibling distance=5em},
    level distance          = 4em,
    edge from parent/.style = {draw, -latex},
    every node/.style       = {font=\footnotesize},
    sloped
  ]
  \node {$x_1$?}
    child  { node {$x_2$?}
        child { node {$d_1$}
        edge from parent node [left, above] {\false}
        }
        child { node {$x_3$?}
            child { node {$d_2$}
            edge from parent node [left, above] {\false}
            }
            child { node {$d_3$}
            edge from parent node [right, above] {\true}
            }
        edge from parent node [right, above] {\true}
        }
    edge from parent node [left, above] {\false}
    }
    child { node {$x_3$?}
        child { node {$d_4$}
        edge from parent node [left, above] {\false}
        }
        child { node {$d_5$}
        edge from parent node [right, above] {\true}
        }
    edge from parent node [right, above] {\true} };
\end{tikzpicture}
}
\end{tabular}
\caption{A decision table (left) and a possible, optimized BDT representing it (right), whose depth is 3}\label{tab:example:DT:BDT}
\end{table}

A BDT is a strategy for testing the conditions that leads to a decision in every possible situation.
In particular, the tree contains a test in each internal node, a label \true/\false in each arc, and a decision in every leaf (possibly repeated). An example of BDT is given in Table~\ref{tab:example:DT:BDT} (right). The procedural interpretation of the BDT is: starting from the root, test the condition in the node, and proceed along the arc labeled with the result of the test. Repeat until a leaf is reached, which represents the decision to be taken. The \emph{cost} of the decision taken is the number of tests performed, which coincides with the length of the path from the root to that leaf. For a given table, there can be several BDT representing it, and BDT minimization is the problem of finding a BDT whose decision costs are all within a maximum, which coincides with the \emph{depth} of the tree (\ie the maximum path length from the root to a leaf in the tree).

The cost could also be averaged over all leaves, yielding a search for the minimum \emph{expected} cost.

The decision problem related to such optimization problem is: ``Given a decision table and an integer $M$, is there a BDT representation of the table whose decision cost is $\le M$?'' This problem was proved NP-complete\footnote{In the general case each test may have its own cost, and each combination of conditions may have a probability attached to it, but the problem is NP-complete even with uniform costs and probabilities, which is the case we are describing here.} for the minimum expected cost by
Hyafil and Rivest \cite{HyafilR76}, and then extended to minimum depth by Chikalov \etal \cite[Prop.~13]{chik-etal-DAM2016}.

To show that  BDT minimization is a particular form of the problem of devising optimal \crs strategies, given a decision table $q \times p$ we fix a catalog \C of $q$ items, each item with $p$ Boolean features.
Decisions are one-one with items, so the number of decisions is $|\C|$, and  conditions/tests  are one-one with  Boolean features. The \crs begins the conversation from a \emph{cold-start state} in which no preference is known yet. More precisely,
\begin{itemize}
    \item the query  contains only variables: $Q_0 = \Vector{x_1,x_2,\ldots,x_p}$;
    \item the constraints vector is a $p$-vector of empty sets $\K_0 = \Vector{\emptyset,\emptyset,\ldots,\emptyset}$;
    \item $\P_0 = \N_0 = \emptyset$.
\end{itemize}
Moreover, for our lower bound we suppose that the answers given by the user always lead to an item that exists in the catalog, that is, $\I \subseteq \C$, so that once the system eventually proposes an item (after a conversation in which the user has the truthful attitude explained above), the user immediately accepts it.
In this situation, the minimal interaction sequences are the ones in which only slot filling is performed, going from the root of the tree to a chosen item in a leaf.
Such conversations composed solely by slot fillings on Boolean features are equivalent to the paths of a BDT, hence a lower bound on the efficiency of such conversations is also a lower bound for minimizing general conversational strategies.

Borrowing from results about the complexity of finding optimal BDTs, we prove below that the problem of finding an efficient \crs strategy is  NP-hard in general, since it includes a decision problem which is NP-complete.

\begin{theorem}\label{thm:np-hardness}
The following problem is NP-hard: ``given an integer $M$,
a catalog \C, an initial cold-start state $U_0 + \Rec(U_0)$, where $U_0 = \Vector{Q_0,\K_0,\P_0,\N_0}$, and $Q_0 = \Vector{x_1,x_2,\ldots,x_p}$, the constraints $\K_0$ are a $p$-vector of empty sets, and $\P_0 = \N_0 = \emptyset$,
is there a conversation strategy made only of slot filling questions, such that the maximum number of interactions with the user is less than or equal to $M$?''
\end{theorem}

\begin{proof}
We give a polynomial reduction from BDT minimization, that we precisely recall here: given a set of objects $O_1,\ldots,O_c$, and a set of tests $T_1,\ldots,T_p$, such that each test returns \true for exactly three objects\footnote{Such a requirement comes from the original reduction \cite{HyafilR76} from the NP-complete problem \textsc{exact-3-cover}. The proof extends to tests selecting more than three objects, but we keep the original special case for the discussion in Section~\ref{subsec:discussion-lower-bounds}.}, decide whether there is a  BDT whose internal nodes are the tests, whose leaves are the objects, and whose depth is less than or equal to a given length $M$.

Given an instance of the source problem as above, we construct the target problem in conversational recommendation. Let the catalog \C be made by $c$ items: for each object $O$, there is an item $I_O$.

We let the Boolean features be one-one with tests
\ie  $f_i = T_i$, for $i=1,\ldots,p$ while the items are defined as $I_O = \Vector{T_1(O),\ldots,T_p(O)}$
\ie each feature $f_i$ of an item $I_O$ has the same value as the one returned by Test $T_i$ on Object $O$.

We now prove that there exists a BDT solving an instance of the source problem, whose depth is less than or equal to $M$ if and only if there is an efficient strategy for the target \crs, whose maximum number of slot filling interactions is less than or equal to $M$.
Without loss of generality, we represent the strategy of the \crs as a binary tree, too.

``$\Rightarrow$'' Suppose there exists such a BDT. Then define an efficient strategy as follows:
take the BDT $T$, and substitute test labels $T_i$ in the internal nodes with feature labels $f_i$ whose filling should be asked to the user, leaving edge connections untouched (\true/\false test results in BDT are one-one with \true/\false user filling in the \crs strategy). Then substitute each leaf node of the BDT, marked with object $O$, with a leaf marked by item $I_O$.

``$\Leftarrow$'' Suppose now that there exists an efficient \crs strategy for the target problem defined above---with $p$ Boolean features and $c$ items---whose maximum number of slot filling questions is within $M$.
Such a strategy defines a  BDT whose maximum path length is $M$ as follows: transform the binary tree representing the strategy by substituting each internal-node feature label $f_i$ with test $T_i$, and each leaf-node item label $I_O$ with object $O$. Clearly such a BDT is isomorphic to the strategy, hence its depth is $M$.
\end{proof}

The above theorem highlights the difficulty in optimizing \crs efficiency in the case in which the only transformation is slot filling; one may wonder how representative of the general case such ``simple'' case is, and whether interactions with slot unfilling/change could be more efficient than slot filling, after all.
However, the following  theorem shows that the other transformations are just ``deroutes'' from the shortest interaction paths; in particular, we show that  a conversation of length $n$, containing all kinds of transformation---\ie slot filling/unfilling/change---can always be transformed into a conversation of length $m$ containing only slot filling, whose length $m$ is no worse than the initial one $n$---\ie $m \le n$.
This shows that the search for shortest transformations can be limited to the ones involving only slot filling.

\begin{theorem}\label{thm:only-slot-filling}
    For every sequence of interactions from $U_0 + \Rec(U_0)$ to a success state $U_n + \{I\}$, where $I$ is accepted by the user, there is another sequence of interactions from a state $U_0' + \Rec(U_0')$ to a state $U_m' + \{I\}$, whose transformations are only substitutions (slot fillings), and whose number of interactions $m$ is not larger than the initial one, \ie $m \le n$.
\end{theorem}
\begin{proof}
If the interaction is already composed by substitutions only, the claim is trivially true; suppose then that there are also slot unfilling and slot change interactions. For every contraction $\sigma(v) = x$, we focus on the state in which the feature contained $v$. If $f_i(Q_0) = v$, \ie the value was in the initial query $Q_0$, let $Q_0' = \tau(Q_0)$ where $\tau(v) = x$, that is, we start with the variable unfilled from the very beginning. If instead the filling happened in an intermediate interaction, we just remove the interaction in which $f_i$ was (erroneously) filled by $v$. Of course, we eliminate also the subsequent slot unfilling interaction. In both cases, the new sequence of interactions is shorter (by 1 or by 2) than $n$. We proceed in a similar way also for slot changing, and repeat the steps until no more such steps are performed. In the end, we obtain an interaction sequence that starts from $U_0' + \Rec(U_0')$ (with $U_0$ possibly different from $U_0$), ends with the same satisfying item at $U_m' + \{I\}$, where $m \le n$.
\end{proof}

The above theorem tells us that the shortest possible conversations include the ones in which only slot filling is performed. Such conversations are the lucky ones, because there is no later revision of the slot fillings made. The goal of the \crs in this case is just to ask the least number of slot fillings leading to one item to propose, and stop. Hence, the above theorem evidences the significance of our NP-hardness reduction for \crs in general. We now make some points about this result.

\subsubsection{Discussion on the lower bound}\label{subsec:discussion-lower-bounds}
We stress that we are not interested in the above proof by itself,
but in how it enriches the understanding of the computational problem of efficiency in \crs, thanks to the following observations:
\begin{enumerate}
    \item The problem is NP-hard for the case in which each \true feature selects three items. This means that NP-hardness does not depend on \emph{feature entropy}, that would be maximal when a feature selects half the catalog. Moreover, the strategy that selects the feature with maximum entropy is only a heuristic, which does not guarantee optimality.
    Note also that since the original reduction \cite{HyafilR76} is from \textsc{exact-3-cover}, and the more general problem \textsc{exact-cover} is NP-hard too, our problem is NP-hard when a feature selects \emph{at least} three items.


    \item Hardness is instead entangled with catalog \emph{incompleteness}, that is, the fact that not all combinations of feature values correspond to an available item; for (although unrealistic) complete catalogs---catalogs in which there is an item for every possible combination of feature values---optimal decision trees can be constructed in time polynomial in the catalog size \cite{guij-etal-IPL99}.

    \item The proof does not mention how the tree is represented, only the decision whether there exists such a tree or not. Hence, NP-hardness holds also in the case in which the tree is represented implicitly, as a circuit (or a neural network) which, given as input a state of the conversation, suggests the next feature to fill (usually called a \emph{succinct} representation \cite{papa-yann-86} of a tree).

    \item Observe that there are at most ${{c}\choose{3}}$ possible different tests, hence $p \in O(n^3)$---\ie the number of features is bounded by a (low) polynomial in $c$. This means that NP-hardness shows up even when the number of features $p$ and the number of items are polynomially related.

    \item The proof does not rely on the existence of more than one user; hence the problem is NP-hard even if we want to devise an efficient strategy only for one particular user (or, equivalently, for non-personalized \crs).

    \item The complexity analysis confirms the intuition  that cold start is a worst case for \crs. When some features in $Q$ contain a value, or \P and \N are not empty, this prunes some parts of the decision tree. Of course, the problem remains NP-hard in the size of the remaining features and items. Many solutions have been proposed to alleviate such a cold-start problem, however, as in many applications new entries are very common (new-user and new-item problems), cold-start situations seem unavoidable, and algorithms that implement efficient strategies should cope with such situations.

    \item Theorems~\ref{thm:np-hardness} and~\ref{thm:only-slot-filling} together show that, although the conversation is composed by actions of slot filling, slot unfilling and slot change, \emph{the problem of devising an efficient strategy is not a problem of optimal planning} of the order of such actions. Slot unfilling and slot change are fundamental for another dimension of \crs evaluation, namely, \emph{effectiveness}: when slot filling comes to a dead end (an item rejection), only the other two actions may steer  the conversation out of that dead end, and lead to an acceptable alternative recommendation.
\end{enumerate}

\subsection{Upper Bounds}\label{sec:upper-bounds}

We now study the upper bound of the complexity of devising efficient strategies for \crs. Since generally feature domains contain several values, we represent slot filling strategies by (general) Decision Trees (DT), which are trees whose internal nodes can have more than two outgoing edges (one edge for each possible feature value).
Rephrasing the definition of BDTs of the previous section, a DT for driving the conversation of a \crs contains a feature $f$ in each internal node, a label with a value $v$ filling $f$ in each arc outgoing from the node, and an item in every leaf.
An example of a catalog and a possible slot filling strategy for the conversation (represented as a DT) are given in Table~\ref{tab:example:DT}.

\begin{table}[htb]
\begin{tabular}{c|c}
\footnotesize
    \begin{tabular}{llll}
\textit{director}    & \textit{starring}  & \textit{genre}     & movie        \\ \hline
\textit{Spielberg}   & \textit{Hanks}     & \textit{historical} & \textit{Forrest Gump} \\
\textit{Spielberg}   & \textit{Dreyfuss}  & \textit{action}    & \textit{Jaws}         \\
\textit{Eastwood}    & \textit{Hanks}     & \textit{action}    & \textit{Sully}
\end{tabular}
\normalsize
     &
\adjustbox{valign=m}{%
\begin{tikzpicture}
[
    level 1/.style={sibling distance=5em},
    level 2/.style={sibling distance=5em},
    level distance          = 5em,
    edge from parent/.style = {draw, -latex},
    every node/.style       = {font=\scriptsize},
    sloped
  ]
\node {director?}
    child { node {starring?}
        child { node {Forrest Gump}
        edge from parent node [left, above] {Hanks}
        }
        child { node {Jaws}
        edge from parent node [right, above] {Dreyfuss}
        }
    edge from parent node [left, above] {Spielberg}
    }
    child { node {Sully}
    edge from parent node [right, above] {Eastwood}
    }
;
\end{tikzpicture}
}
\end{tabular}
\caption{A catalog (left) and a possible, optimized, conversational strategy (right) represented as a DT. Note that two features (out of three) are sufficient to distinguish all items.}\label{tab:example:DT}
\end{table}

The strategy represented by the DT is: starting from the root, ask the user to fill the feature in the node with a preferred value, and proceed along the arc labeled with the user's answer. Repeat until a leaf is reached, which identifies the item to propose. The \emph{maximum cost} of such a part of the conversation is the maximum number of slot fillings asked to the user---\ie  the depth of the tree.
We recall that the problem of finding an optimal decision tree (in terms of minimum depth) is trivially NP-complete: it is NP-hard, because it contains BDT optimization (see previous section) as a special case, and it belongs to NP because the DT representing the strategy is polynomial in the size of the catalog. Hence a nondeterministic Turing Machine 
can compute such an optimal DT in polynomial time, making nondeterministic choices about which feature labels each internal node.

Observe that at any given point of the conversation, the focus is restricted to a set of \textbf{recommendable items} $\S \subseteq \C$ compatible with the answers given by the user so far; for example, if a user tells the system that one of her most favorite directors is \textit{Spielberg}, the system concentrates on Spielberg's movies (the set \S) excluding (for the moment) the others. In such a set of items, $f_{\mathit{director}} = \mathit{Spielberg}$,
and each other feature $f_i$ may contain only a subset of all possible values in its domain $D_i$---only values compatible with Spielberg's movies.
Continuing the example, for the items in \S the feature $f_{\mathit{starring}}$ may contain a value among \textit{Hanks, Duvall, Cruise, etc.}, but it never contains \textit{De Sica}, because there are no movies starring \textit{De Sica} whose director is Spielberg. We call such compatible values \emph{Active Values} for $f_i$ in \S, written as $AV(\S,f_i)$. More formally,
$AV(\S,f_i) = \{ v \in D_i ~|~ \exists I \in \S : v = f_i(I) \}$.
In the previous example, once the user's answers restrict the movies to the set \S of the movies directed by Spielberg, $AV(\S,f_{\mathit{starring}}) = \{ \mathit{Hanks, Duvall, Cruise} , \ldots\} $. Active values will be useful both in the rest of this theoretical section, and in the next section regarding experiments.

An algorithm computing a DT for the catalog \C could be used in \crs for a strategy using \emph{only} slot filling.
However, a single DT would be an incomplete strategy for \crs, since it does not consider the possibility of a reject action by the user, and consequently, slot unfilling and slot change.
On the other hand, observe that the possible actions of the \crs cannot be mixed in any possible way: infinite loops---such as a slot filling followed immediately by an unfilling of the same slot---must be avoided. Hence, in what follows we make a reasonable restriction on possible strategies: every slot unfilling (contraction) or slot change (revision) can be performed only if, immediately before them, at least one item has been rejected by the user. Such an item can therefore be added to the set \N the user dislikes, and will not be proposed again. Recalling that the \crs can propose only items in $\C - \N$,
this incrementality in \N mirrors a reduction of the items to propose, and leads to a trivial upper bound on the possible strategies: each item could be proposed at most once.
After all possible items in the catalog have been proposed and rejected, the conversation must stop. We call such strategies \emph{well founded}, in analogy with well-founded sets, and we restrict ourselves to them in the analysis below.

Intuitively, deciding whether or not there exists a well-founded strategy for a catalog \C, which ends in at most $M$ interactions, is a problem that can be solved by exploring recursively all possible transformations, and all active values for a feature in a slot filling or slot changing transformation. However complex, such an exhaustive exploration needs only polynomial space (the size of the catalog times the size of feature domains), and would prove that the problem of devising an optimal strategy is in PSPACE. We prove this intuition by first exhibiting Algorithm~\ref{algo-upper-bound} below, and then analyzing its space complexity.

\renewcommand{\algorithmicrequire}{\textbf{Input:}} 
\renewcommand{\algorithmicensure}{\textbf{Output:}}

\begin{theorem}[Correctness]\label{thm:correctness}
Given as input a catalog \C of items, an integer $M$, and an initial state $U_0 + R(U_0)$, Algorithm~\ref{algo-upper-bound} returns \true if and only if there exists a well-founded strategy for the \crs which, starting from $U_0 + R(U_0)$, ends in at most $M$ interactions.
\end{theorem}
\begin{proof}
First of all, observe that Algorithm~\ref{algo-upper-bound} terminates, since in every recursive call either the number $M$ decreases, or the set \N of disliked items enlarges, or both, leading necessarily to a base case.
Once termination is proved, correctness of Algorithm~\ref{algo-upper-bound} can be proved by induction on $M$. For the base cases $M \le 0$ (Line~\ref{alg:base-case-1}), $|\C-\N| \le M$ (Line~\ref{alg:base-case-2}), and user acceptance (Line~\ref{alg:base-case-3}), the result is obviously correct (see also comments aside base cases). Suppose now that the algorithm is correct for any state and interaction bound less than $M$.

To ease the proof presentation, note that the recursive part of Algorithm~\ref{algo-upper-bound} can be divided in two branches: first one spans Lines~\ref{alg:branch1-begin}--\ref{alg:branch1-end}, while second one spans Lines~\ref{alg:branch2-begin}--\ref{alg:branch2-end}. We refer to these two branches in the rest of the proof.
\vskip 1\baselineskip
\noindent
``$\Rightarrow$'' Suppose the algorithm terminates with \true on input $U+R(U)$ and $M$. Either this result comes from the first branch, or from the second one; for both branches, in the lines marked with $(\exists)$, the algorithm chose a feature $f$---and possibly a transformation in case of the first branch---such that for any possible value chosen by the user (in the lines marked with $(\forall)$), all recursive calls\footnote{Strictly speaking, when the transformation is a slot unfilling, there is no choice by the user, and just one recursive call returning \true. In this case the inductive hypothesis on $U'+R(U'), M-1$ directly proves that there is a strategy also for $U+R(U), M$.} terminate with \true on input $U'+R(U')$ (where $U'$ depends on the value chosen by the user) and $m \in \{ M-1, M-2\}$. By inductive hypothesis, each recursive call ensures the existence of a well-founded strategy for the \crs which, given $U'+R(U')$, ends in at most $m$ interactions.
Then, the overall strategy on input $U+R(U), M$ is the one that makes exactly the choice made in  $(\exists)$, followed by the strategies defined by the recursive calls.
\vskip 1\baselineskip
\noindent
``$\Leftarrow$'' Suppose the algorithm terminates with \false. This means that in either branch, for any possible choice (feature, transformation) made in $(\exists)$-lines, there is at least one possible choice (feature value $v$) the user can make in ($\forall$)-lines that makes at least one recursive call end with \false on input $U'+R(U'), m\in \{ M-1, M-2\}$.
By inductive hypothesis, this means that when the user chooses $v$, there is no strategy starting from $U'+R(U')$ that ends in at most $m$ interactions. Hence there is no strategy starting from $U+R(U)$ and ending in $M$ interactions, either.
\end{proof}

\begin{theorem}\label{thm:upper-bound}
Given as input a catalog \C of items---with $p$ features and at most $K$ values in each feature domain---an integer $M$, and an initial cold-start state $U_0 + R(U_0)$, Algorithm~\ref{algo-upper-bound} uses polynomial  space.
\end{theorem}
\begin{proof}
As in the previous proof, we refer to Lines~\ref{alg:branch1-begin}--\ref{alg:branch1-end} as ``first branch'', while ``second branch'' refers to Lines~\ref{alg:branch2-begin}--\ref{alg:branch2-end}.

First of all, a single recursive call needs a counter of $\lceil \log p \rceil$ bits for enumerating features (either Line~\ref{alg:enum-features1} or Line~\ref{alg:enum-features2}, depending on which branch) and a counter of $\lceil \log K \rceil$ bits for enumerating values in the chosen feature (lines immediately below the above ones).
Hence, the space of a single recursive call is $O((\log p) \cdot (\log K))$.

A bound on the height of the call stack can be computed as follows: at most $p$ nested calls can be made in the second branch---filling all $p$ features in $Q$---after which a single item has been isolated, and the first branch must be taken. Each time the first branch is taken, at least one item is added to \N (in the worst case), reducing the set of items $\C-\N$ that can be proposed. Clearly, at most $|\C|$ calls can recur into the first branch. Overall, the recursion stack height is bounded by $p\cdot |\C|$, so the total space occupied by the recursion stack is $O(p\cdot |\C| \cdot (\log p) \cdot (\log K))$.

\begin{algorithm}[]
\scriptsize
\caption{ExploreStrategies($\C,U,M$)}\label{algo-upper-bound}
\begin{algorithmic}[1]
\Require catalog \C (fixed), conversation state $U = \Vector{Q,\K,\P,\N}$, integer $M$
\Ensure Boolean
$ \left\{\begin{array}{ll}
   \true  & \mbox{if there is a strategy whose number of interactions is } \le M \\
   \false  & \mbox{otherwise}
\end{array} \right.$
\If{$M \le 0$}\label{alg:base-case-1}
    \Return \false \myComment{Base case 1: there is no strategy with 0 interactions}
\Else
\If{$|\C-\N| \le M$}\label{alg:base-case-2}
    \Return \true \myComment{Base case 2: trivial strategy, propose each item}
\Else[$|\C-\N| > M > 0$]
\If{$Q(\C-\N) = \{I\}$ } \myComment{$Q$ selects a singleton}
    \State\label{alg:branch1-begin} propose I to the user
    \If{$I$ is accepted}\label{alg:base-case-3}
        \Return \true \myComment{item found within $M$ interactions}
    \Else[$I$ is rejected]
        \State let $\N' \gets \N \cup\{I\}$
    \State \emph{``the user rejects $I$ for a specific reason (value)?''}
    \If{$\exists f_i$ s.t.\ $f_i(I)=v$ \textbf{and} the user dislikes $v$}
    \label{alg:feature-value-reject}
        \State let $C_i' \gets C_i \cup \{v\}$ \myComment{$v$ is added to disliked values}
        \State let $\K' \gets \Vector{C_1,\ldots,C_i',\ldots,C_p}$
        \State $\N' \gets \N' \cup \{I' ~|~ v = f_i(I') \}$
            \myComment{add items with $v$ to disliked items}
    \Else
        \State $\K' \gets \K$
    \EndIf \myComment{no more items to propose at this point, change $Q$}
    \State\label{alg:enum-features1} $(\exists)$ choose a feature $f_i$ s.t.\ $f_i(Q) = v$ ($f_i$ was filled in $Q$)
    \State $(\exists)$ choose $\tau \in \{$ Slot Unfilling (SU), Slot Change (SC)$\}$
    \If{$\tau$ is SU}\label{alg:SU}
        \State let $\tau(v) = x$ where $x \in X$ is a new variable
        \State \Return ExploreStrategies($\C,\Vector{\tau(Q),\K',\P,\N'}, M-1$)
        \State \myComment{$-1$ interaction: the rejection}
        \Else[$\tau$ is SC] \myComment{try all values but $v$ and the disliked ones}
            \State $(\forall)$ let $tryAllSlotChanges \gets \true$
            \ForAll{$v'\in D_i$ s.t.\ $v' \not\in C_i \cup \{v\}$}
                \State let $\tau(v) = v'$
                \If{$\tau(Q)(\C-\N)\neq \emptyset$} \myComment{try only changes selecting some item}
                \If{
                \parbox[t]{.45\textwidth}{ExploreStrategies($\C,\Vector{\tau(Q),\K',\P,\N'}, M-2$)}}
                \myComment{$-2$ interactions: rejection, and change value)}
                \State returns \false
            \State $tryAllSlotChanges \gets \false$
            \myComment{user might choose a value leading to $>M$ interactions}
           \EndIf\EndIf
        \EndFor
        \State \Return $tryAllSlotChanges$ \myComment{\true only if any user's change leads to $\le M$ interactions }
    \EndIf\EndIf\label{alg:branch1-end}
    \Else[$Q(C)$ selects more than one item]
    \label{alg:branch2-begin}
        \myComment{ask the user to fill a slot that distinguishes items}
        \State\label{alg:enum-features2} $(\exists)$ choose a feature $f_i$ s.t.\ $f_i(Q)=x \in X$
        \State $(\forall)$ let $tryAllSlotFillers \gets \true$
        \ForAll{$v \in AV(Q(\C),f_i)$}
        \myComment{try all active values occurring in items of $Q(\C)$}
        \State let $\tau(x)=v$
        \If{\parbox[t]{.45\textwidth}{ExploreStrategies($\C,\Vector{\tau(Q),\K,\P,\N}, M-1$)\\ returns \false}}
             \myComment{one interaction less (slot filling)}
            \State $tryAllSlotFillers \gets \false$
            \myComment{user might fill a value leading to $>M$ interactions}
        \EndIf
        \EndFor
        \State \Return $tryAllSlotFillers$
        \myComment{\true only if any user's fill leads to $\le M$ interactions }
    \EndIf\label{alg:branch2-end}
\EndIf
\EndIf
\end{algorithmic}
\end{algorithm}

Observe now that an item with $p$ features (each feature containing one value among $K$ possible ones), needs $p \cdot\lceil \log K \rceil$ bits of encoding, hence the input of a catalog of $|\C|$ items has size $n = |\C| \cdot p \cdot\lceil \log K \rceil$. With respect to such an input size, the stack space is $O(n \cdot \log p)$. Now from the product expressing $n$, it follows that $\log p < \log n$, hence $O(n\cdot\log p) \subseteq O(n\cdot\log n)$. This proves that Algorithm~\ref{algo-upper-bound} uses polynomial space.
\end{proof}

Combining Theorems~\ref{thm:correctness} and~\ref{thm:upper-bound}, one obtains that the problem of deciding whether there exists a strategy for \crs terminating in a given number of interactions can be solved in Nondeterministic Polynomial Space (NPSPACE), and since NPSPACE = PSPACE \cite{savi-JCSS70}, the problem belongs to PSPACE, too. Recall that $PSPACE \subseteq EXPTIME$, and in fact, it can be easily seen that Algorithm~\ref{algo-upper-bound} runs in exponential time in the worst case.

\subsection{Improving the upper bound}\label{sec:imp_upB}

The above lower and upper bounds restrict the complexity of finding efficient strategies for \crs between NP (lower bound) and PSPACE (upper bound).
Since the aim of this paper is to bridge \crs practical efficiency with theoretical results, we leave investigation on how to close this gap to future (more theoretically-oriented) research.
On the practical side, our PSPACE upper bound ensures that usual techniques for solving NP-complete problems, and in particular, finding optimal DT---\eg Mixed-Integer Optimization \cite{bert-dunn-ML17}, branch-and-bound search \cite{agli-etal-AAAI2020}, SAT and MAXSAT encoding \cite{naro-etal-IJCAI18}---can be used for devising optimal conversational strategies, but only---as far as our result proves---as steps inside an overall recursive  procedure,
while our NP lower bound ensures that employing such techniques is not an overshoot.

To improve the upper bound, remark that when the user rejects a particular feature value---\eg ``\textit{I don't like green phones}'', as taken into account in Line~\ref{alg:feature-value-reject}---this fact rules out all items sharing the same feature value, thus improving the worst case of the strategy. Let us call Protocol P1 the protocol in which the \crs does not ask for disliked values when an item is rejected (\ie  Algorithm~\ref{algo-upper-bound} never enters Line~\ref{alg:feature-value-reject}), and call Protocol P2 the one in which the \crs \emph{always} asks for a disliked value when the user rejects an item, the user provides such a value, and the \crs discards all items sharing the same feature value.

Intuitively, it seems that P2 is always more efficient than P1. We prove in what follows that, in fact, the difference between P1 and P2 carries over to an improvement in the upper bound of the efficiency problem, backing up intuition, but highlighting when such intuition fails.
For this part, we recall $\theta$-notation: for a function $f(n)$, we write  $f(n) \in \theta(g(n))$ if both $f(n) \in O(g(n))$, and $f(n) \in \Omega(g(n))$, \ie $f(n)$ grows no faster than $g(n)$, but also there are infinitely many (worst) cases in which $f(n)$ grows as fast as $g(n)$. Recall that the relation ``$\in\theta$'' between two functions is symmetric and transitive\footnote{Simmetry: $f(n) \in \theta(g(n)) \Longleftrightarrow g(n) \in \theta(f(n))$. Transitivity: $f(n) \in \theta(g(n)), g(n) \in \theta(h(n)) \Rightarrow f(n) \in \theta(h(n))$}.
\begin{corollary}\label{corollary:polylogspace}
Let the given data as in Theorem~\ref{thm:upper-bound}, and additionally, suppose the number of values of all feature domains be bounded by a constant $K$, and $|\C| \in \theta(K^p)$. Then deciding whether there is a well-founded strategy always terminating within $M$ interactions, that applies Protocol P2, is a problem that belongs to POLYLOGSPACE.
\end{corollary}
\begin{proof}
We prove membership in POLYLOGSPACE using the same Algorithm~\ref{algo-upper-bound} above, plus the fact that in every iteration of the algorithm the condition in Line~\ref{alg:feature-value-reject} is satisfied, as required by Protocol P2.

Observe that $|\C|  \le K^p$ since there cannot be more items than all possible $p$ combinations of feature values. After $K-1$ rejections in Line~\ref{alg:feature-value-reject}, each ruling out a value for a feature, the number of remaining items $|\C - \N |$ is bounded by $K^{(p-1)}$ (all but one possible values for a feature have been eliminated, hence that feature does not count in the exponent). When the values ruled out $v_{i_1}, v_{i_2},\ldots$ belong each one to a different feature $f_{i_1}, f_{i_2}, \ldots$, the bound  is a more complex
product\footnote{For example, if the disliked values were 2 for the first feature, 5 for the third one, and 3 for the last feature, the expression would be the product of $p$ factors $(K-2)\cdot K \cdot (K-5) \cdot K \cdots K \cdot (K-3)$. } than $K^{(p-1)}$, but in any case, after at most $p\cdot (K-1)$ rejections and disliked values, all values but one have been ruled out for each feature domain. At that point, either the user accepts the last item---the one whose features contain the last remaining value in each domain---or the user rejects it, and in any case the conversation concludes. This means that now the stack height of Algorithm~\ref{algo-upper-bound} is no more bounded by $p\cdot |\C|$ (where the factor $|\C|$ was the bound on possible rejections) as in the proof of Theorem~\ref{thm:upper-bound}, but is bounded by $p \cdot p(K-1) = p^2(K-1)$, and the space needed by the call stack is $O(p^2(K-1) \cdot (\log p) \cdot (\log K)) = O(p^2 \log p)$ for constant $K$.
We now prove that the above expression implies that space requirements drop to a polynomial in the \emph{logarithm} of the input.

Observe that $|\C| \in \theta(K^p)$ implies by symmetry $K^p \in \theta(|\C|)$, that is, $p \in \theta(\log |\C|)$ when $K$ is a constant. Hence, by transitivity the number of nested recursive calls is in $\theta(\log^2 |\C|)$, and the space needed by Algorithm~\ref{algo-upper-bound} drops to $\theta((\log^2 |\C|) (\log\log |\C|)) \subset O(\log^3 |\C|)$.
The size of the input is $n = |\C| \cdot p \cdot\lceil \log K \rceil$, which belongs to $\theta(|\C| \cdot \log |\C|)$ in the hypotheses of this theorem, hence $\log n \in \theta(\log |\C|+\log\log|\C|) = \theta(\log |\C|)$.
Summing up, the space needed by nondeterministic Algorithm~\ref{algo-upper-bound} is in $O(\log^3 n)$, which, by Savitch's theorem \cite{savi-JCSS70}, implies that the problem complexity now drops to $DSPACE(\log^6 n)$ (problems solvable by a deterministic algorithm that uses $O(\log^6 n)$ space), which is strictly included in POLYLOGSPACE.
\end{proof}

However complex the class $DSPACE(\log^6 n)$ might seem, observe that the upper bound was exponentially reduced with respect to PSPACE. Moreover, from the well-known relation $POLYLOGSPACE \subseteq DTIME(2^{\lceil \log n \rceil ^{O(1)}})$, (problems solvable by a deterministic algorithm whose time is bounded by a subexponential function, a class which is also known as \emph{Quasi-Polynomial time (QP)}\ ), we note that also time upper bounds provably improve over the previous EXPTIME ones\footnote{Observe that the general problem remains NP-hard, though, because the additional hypotheses of this corollary are not met by the reduction in Theorem~\ref{thm:np-hardness}.}.

Hence, the above result backs up theoretically \crs  implementations whose strategies enforce Protocol~P2.
However, observe that the above result holds only in the situation in which the input catalog has a number of items that grows as fast as $K^p$---\ie items are ``dense'' in the space of features---and $K$ is a constant, hence the input increases because more features $f_{p+1}, f_{p+2},\ldots$ are added---\ie more domains $D_{p+1}, D_{p+2}, \ldots$ are added---not because each domain $D_i$ of each feature value increases its cardinality. If, instead of considering $K$ as a constant, we would consider $p $ as a constant, the improvement on the upper bound due to Protocol~P2 would not theoretically show up. In other words, when the number $p$ of features is fixed, and the catalog increases because more possible values are added to each feature, asking to rule out such \emph{values} one by one (as in Protocol~P2) is not provably more efficient than asking to rule out \emph{items} one by one (as in Protocol~P1).

Hence, our theoretical analysis points out which characteristics of the catalog must be evaluated when implementing efficient  strategies for a \crs: if the catalog is composed by items with few features, and lots of possible values in each feature, there is not yet any theoretical evidence that Protocol~P2 yields a tangible increase in efficiency with respect to Protocol~P1; if instead the catalog consists of items described by lots of features, each feature with  few values, then~P2 improves efficiency over~P1 in a complexity-theoretical provable way.
In the next section, we provide an evidence of such a difference in a practical setting. Namely, we set up two catalogs IS1, IS2---both derived from real data on movie recommendation---with the characteristics described above, and simulate conversations with users---whose preferences are taken from real user preferences---testing both protocols. The results confirm the theoretical expectations, namely, that the additional steps of Protocol~P2 really pay off only for the second catalog, IS2.

\section{A complexity-driven experiment with two protocols and two catalogs} \label{sec:experiment}

Based on the results of the previous section, we devise an \textit{in-vitro} experiment that confirms the intuition in Corollary~\ref{corollary:polylogspace}.
More specifically, we demonstrate that the two alternatives in Algorithm~\ref{algo-upper-bound}---that is, entering or not  Line~\ref{alg:feature-value-reject}---lead to different results in terms of the efficiency of a \crs, but depending on the characteristics of the catalog. We recall the two protocols below, then we make our hypotheses explicit, and explain in detail the setting of the experiment. Subsequently, we present some possible conversations based on the protocols, we discuss the outcomes (confirming the hypotheses) and their limitations.

\subsection{Protocols}
A user can take one of two distinct actions when she rejects a recommendation, and each action corresponds to one of the protocols P1, P2, below:
\begin{itemize}
    \item[P1] - the user rejects the recommendation and the CRS does not ask the user to provide a specific reason, \ie a reason that refers to a disliked feature value. Examples of such more unspecific feedback---if any feedback is given at all---could be, ``\textit{I don't want to go to the Green Smoke restaurant}'' (maybe because a friend of mine reported privately to me a bad impression), or ``\textit{I don't want to see the movie \textit{American Beauty}}'' (because the title sounds strange to me, but I cannot explain this to a system);
    \item[P2] - the user rejects the recommendation and the CRS asks for a specific item characteristic (\ie feature value) she does not like at all. For example, green color for cellphones, sea-view restaurants, a particular movie director, etc. We assume that a user will truthfully answer such a question when asked.
\end{itemize}
We will describe the details and assumptions of how we simulate the conversations below.

\subsection{Research Hypotheses}
We formulate two hypotheses for our experiment aiming to confirm the theoretical analysis provided at the end of Section~\ref{sec:imp_upB}.

\begin{itemize}
    \item[H1] We do \emph{not} expect a strong difference in terms of efficiency between P1 and P2 when the items in the catalog have few features with a large number of distinct values.
    \item[H2] We \emph{do} expect a strong difference in terms of efficiency between P1 and P2 when the items in the catalog have several features with few distinct values (\ie IS2). 
\end{itemize}

Basically, the only difference between H1 and H2 is the item catalog, which is the independent variable in our experiment.

\subsection{Experimental setting}
We designed an offline 
experiment that simulates the two above-mentioned interaction strategies (\ie protocols) within a \emph{system-driven} model: the \crs asks, the user answers. Examples of corresponding dialogs are outlined in Section \ref{sec:running_ex}.

\subsubsection{Efficiency Metric}
In order to evaluate interaction efficiency we count the number of questions (\emph{NQ)} the \crs asks before the user accepts a recommendation. A smaller number of required questions indicates higher interaction efficiency. This metric, as mentioned above, is commonly used in the CRS literature \cite{iovine2020conversational}.

\subsubsection{Dataset and Catalog Description}
In our experiments, we rely on the MovieLens-1M dataset, which is widely used in recommender-system research, and which contains movie ratings by a larger user community\footnote{\url{https://grouplens.org/datasets/movielens/}}. The statistics of the dataset are given in Table \ref{tab:dataset}.
Unfortunately, in this dataset the items do not have textual features associated to them. Therefore, we mapped the MovieLens dataset with DBpedia resources to obtain side information for each item in the catalog. To close this gap we follow the approach of Anelli \etal \cite{DBLP:conf/semweb/AnelliNSRT19}.
\begin{table}[!ht]
\centering
\caption{Dataset characteristics}\label{tab:dataset}
\begin{tabular}{
	l |
	r |
	r |
	r |
	r 
}
\toprule
{Dataset} & {users} & {items} & {ratings} & {content features} \\
\midrule
Movielens-1M & 6,040 & 3,706 & 1,000,209 & - \\ 
Movielens-1MC (with content) & 6,040 & 3,308 & 511,260 & 279  \\

\bottomrule
\end{tabular}
\end{table}

After this processing, since not all the items have a mapping to the knowledge graph, the number of items became 3,308, and the total number of ratings after the mapping is 511,260. Through the mapping, 279 features were introduced into the dataset. Each item has one or more of these features expressed by a triple: $\langle$\textit{Item}, \textit{Feature}, \textit{Value}$\rangle$. As an example, consider the movie \textit{The Hateful Eight} directed by \textit{Quentin Tarantino}. For this movie, we have the following triple: $\langle$\textit{The Hateful Eight}, \textit{director}, \textit{Quentin Tarantino}$\rangle$.
We denoted this data set as \emph{MovieLens-1MC - with content} in Table \ref{tab:dataset}. 

In order to prove our hypotheses, we built two different itemsets from MovieLens-1MC:
\begin{itemize}
\item \emph{Itemset1 (IS1)} has only a few features, but with a larger number of distinct values and is designed to prove H1 (we do not expect a strong difference in terms of efficiency between P1 and P2 \textit{when the items in the catalog have few features with a large number of distinct values}).
\item \emph{Itemset2 (IS2)}, in contrast, has a larger number of features, but each of them only has a few distinct values and is designed to prove H2 (we expect a strong difference in terms of efficiency between P1 and P2 \textit{when the items in the catalog have several features with few distinct values}).

\end{itemize}
Specifically, to make the datasets sufficiently different and that reflect the characteristics described in our research hypotheses, IS1 and IS2 have 4 and 10 features respectively (selected from the 279 content features of Movielens-1MC) for each item. For each feature of IS1 there are from 1,500 to 2,500 distinct values, whereas there are about 100 for IS2.
The characteristics of the two itemsets are reported in Table~\ref{tab:dataset_statistics}.
\begin{table}[!h]
\centering
\caption{Itemset characteristics. Each value indicates the number of distinct value belonging to each features}\label{tab:dataset_statistics}
\resizebox{\textwidth}{!}{\begin{tabular}{
	l |
	r |
	r |
	r |
	r |
	r |
	r |
	r |
	r |
	r |
	r
}
\toprule
 & {director} & {starring} & {producer} & {writer} & {distributor} & {musicComposer} & {language} & {narrator} & {basedOn} & {country}\\
\midrule
Itemset1 & 1542 & 2293 & 1587 & 1981 & - & - & - & - & - & -\\ 
Itemset2 & 100 & 100 & 100 & 100 & 99 & 100 & 99 & 100 & 100 & 65\\ 
\bottomrule
\end{tabular}}
\end{table}

In order to focus on the main goal of our experiment---to provide evidence that the efficiency of different strategies depends on itemset characteristics---we make the same assumptions as in Section \ref{sec:formalization} that \emph{(i)} values for all features exist and \emph{(ii)} that all features have singleton values. Accordingly, we replaced \texttt{null} values with a value randomly chosen from the set of possible values for that feature, and set-valued features (\eg the movie cast) with a single value randomly chosen among them. Note that for the purpose of our experiment, this replacement does not effect the final result. 

\subsubsection{Experimental Protocol---Details}
In our experiment, we simulate conversations between a user and a \crs where this latter drives the interaction by asking questions to the user about preferred item features and by making recommendations for items. We assume that the simulated user has certain pre-existing preferences regarding item features, and truthfully responds to the system's questions about these preferences. When provided with a recommendation, the user either rejects it, which means that the dialog continues, or accepts it and the dialog ends. The \crs in our simulation implements one of the described conversation strategies, P1 or P2. Remember that---depending on the catalog characteristics---we expect to see differences in terms of how many interaction steps are necessary before the dialog successfully ends.

For the sake of the simulation, we assume that \emph{in each conversation} there is exactly one distinguished item $\hat{I}$ in the catalog $C$ that the user will accept, which we call \emph{test} item or \emph{ideal} item. This item $\hat{I}$ matches all preferences of the user (even if it is not the only item in $C$ that matches the pre-existing general user preferences). Remember here the pre-existing user preferences consist of a number of desired value for features. With respect to the \emph{genre} feature, a user might for example have a preference for \emph{action} and \emph{romantic} movies. When asked by the system for a particular genre preference, the user might therefore state one of the two values. During the conversation, the recommender will only consider such a \emph{stated preference} when retrieving suitable items \S (as defined in Section \ref{sec:upper-bounds}).

In our experiment, we aim to simulate a realistic situation where users do not know before the recommendation process which item (movie) they want \emph{exactly}. Translated to our setting, this means that the user does not know about $\hat{I}$ or its features \emph{in advance}, but is able to tell that she accepts it once it is presented to her based on the stated preferences. As a result, the conversation with the system might often reach a dead end, \ie where the system does not recommend $\hat{I}$ and the user therefore has to revise their \emph{stated} preferences. In our example, the user might have stated to prefer \emph{action} movies when asked by the \crs, even if the ideal item $\hat{I}$ actually is a romantic movie. Hence, the user has to revise the stated preference at some stage so that the ideal item can be recommended by the \crs.\footnote{Note that the chosen protocol will by design lead to a high number of required interactions, in particular as make the specific assumption for our simulation that only one item is acceptable in each dialog.}
Please note that in our experiment we simulate a user cold-start situation, \ie we are not taking any long-term user profile into account the during the dialog.
Without loss of generality, to avoid the introduction of further notation, in this section we will use $f_i$ to refer to the feature selected at the $i$-th step of the dialog. We also remember that during the dialog, features are selected randomly.

In order to determine the pre-existing user preferences and the ideal items based, we select a set of positively-rated items (PRI) for each user. This set consists of those items that the user has rated with a value that is greater or equal to their average rating. We use this set PRI for two purposes. First, we simulate a dialog for each element $I$ of PRI as an ideal item. Second, we use the items in PRI to determine the pre-existing preferences of a user and simulate their answers to the questions posed by the system. The user preferences on the feature $f_i$ are then encoded in the set $UP_i = \bigcup \{ v = f_i(I) \mid I \in \text{PRI}\}$ (for \textit{U}ser \textit{P}references). Therefore, if the user previously liked \textit{action} and \textit{romantic} movies, the set of pre-existing preferences for that feature contain exactly these two values.
Remember that since we want the dialog to be successful only in case it ends with the acceptance of $\hat{I}$, for each conversation the catalog consists of $\mathcal{C}' = (\mathcal{C} - \text{PRI}) \cup \hat{I}$. We remove all the items in PRI but $\hat{I}$ from the catalog $\mathcal{C}$ since we do not want to recommend items the user has already seen before. 


When the simulated dialog with a defined ideal item $\hat{I}$ starts, the system will ask a question on a feature \eg ``\textit{What is your favorite genre?}''. The simulated user will then respond by choosing a value from the set $UP_i \cap AV(\mathcal{S},f_i)$, where $AV$ is the set of \emph{Active Values} (see Section \ref{sec:upper-bounds}). In our example, $f_i = genre$.
By definition, this set of values therefore changes during the interaction as it is obtained through the intersection between all the pre-existing preferences for that feature ($UP_i$) and the possible values in the set of recommendable items $\mathcal{S}$. 
As a consequence, in our simulation the user cannot answer with a value that is not present in any recommendable object in $\mathcal{S}$.

After each user answer, the set of recommendable items $\mathcal{S}$ is updated by the \crs. If the user, for example, answered to prefer action movies, all movies with a different genre are not recommendable anymore at the next stage of the dialog and are then removed from $\mathcal{S}$. We remember that the user does not know about the ideal item $I$ in this session nor its features.
Therefore, the choice of the user answer $v \in UP_i \cap AV(\mathcal{S},f_i)$ is done in a randomized manner in our experiment. This also avoids the introduction of any bias or any optimization.

A recommendation is shown when the system has no more questions to ask. This may happen in two cases:
\begin{enumerate}[label=\alph*)]
    \item The system has asked the user $p$ questions, with $p$ being the number of features;
    \item After the user's answer to the $q$-th question, with $q < p$, we have  $|\mathcal{S}| = 1$. In this case, the system recommends the only item in $\mathcal{S}$ without asking any explicit preference for the features $f_i$ ---with $q < i \leq p$--- to the user. In other words, the systems does not try to elicit preferences for the remaining features that have not been asked yet.
\end{enumerate}
The user rejects the recommendation when $I$ is not present in the list of recommended items.
If the recommendation is rejected, the recommended items are removed from the catalog $\mathcal{C'}$ and the system starts again posing questions to the user. It is noteworthy that, since we may have more than one item in $\mathcal{C'}$ described by the feature values selected by the user during the dialog, the final recommendation may contain also a set of items.

In case of rejection P1 and P2 behave differently. In case of P1, the system starts again from a first randomly selected feature and discards all the values previously selected for each $f_i$ by the user during the previous interactions. In protocol P2, in contrast, the user declares one of the feature values she \emph{dislikes} in the recommended items. Once the user declares she dislikes a value $v$ for a feature $f_i$, the system keeps all the previous choices in the sequence of features $f_1,\ldots, f_{i-1}$ for the current dialog and starts asking questions for $f_i$ and the following ones $f_{i+1}, \ldots, f_p$. 
We cannot keep the value previously selected in the sequence of features $f_{i+1}, \ldots, f_p$ for the current dialog since after the user selects a new value $v' \in UP_i \cap AV(\mathcal{S},f_i)$ for $f_i$, the set $\mathcal{S}$ changes and, as a consequence, the same happens to $UP_{i+1} \cap AV(\mathcal{S},f_{i+1})$ for $f_{i+1}$ (and the remaining features). 
Regarding feature $f_i$, the user is allowed to select values from $UP_i \cap AV(\mathcal{S},f_i) - v$. We observe that protocol P2 always enters Line~\ref{alg:feature-value-reject} in Algorithm~\ref{algo-upper-bound}.



On the other side, when the recommendation succeeds ---when the ideal item $I$ is in the list of recommendations--- the dialog is successfully ended and the simulation continues with a new dialog for another user and/or target item. The simulation ends when a dialog was simulated for each element in PRI for every user.

\subsection{Running Examples}\label{sec:running_ex}
We show the main differences between the two protocols defined above using two fictitious dialogs.


\noindent\begin{example}[Protocol 1 - The user rejects a recommendation and the item is discarded from the dataset.]
~\newline
\begin{description}
\item[System:] Tell me a film score composer you like,  I have these options: \textit{Ennio Morricone}, \textit{Alessandro Alessandroni}, \textit{Paul Sawtell}, \textit{Frank Skinner}.
\item[User:] I like Ennio Morricone.
\item[System:] Who is one of your most liked actors, in your opinion among \textit{Leonardo Di Caprio}, \textit{Clint Eastwood}, \textit{Gian Maria Volonté}?
\item[User:] I like Ennio Morricone.
\item[System:] Do you like \textit{Quentin Tarantino} as film director?
\item[User:] Yes, I like Quentin Tarantino.
\item[System:] Based on your answers,  I recommend you \textit{Django Unchained}. What do you think about it?
\item[User:] Sorry, I don't want it---a friend of mine spoiled the plot for me.
\end{description}
\noindent \textup{The recommendation failed. \textit{Django Unchained} is removed from the set of recommendable items and the conversation goes on. After several failures the conversation may continue as in the following conversation.}
\begin{description}
\item[System:] Who is one of your most loved film score composers among \textit{Piero Umiliani}, \textit{Armando Trovajoli}, \textit{Piero Piccioni}?
\item[User:] I like Piero Umiliani.
\item[System:] Do you like \textit{Mario Monicelli} as film director?
\item[User:] Yes, I do. I like Mario Monicelli.
\item[System:] For the above choices, I have movies of the following genres: \textit{Comedy, Neorealistic, Historical}. Which one would you prefer?
\item[User:] Let's try comedy movies.
\item[System:] Which actor/actress would you prefer most among \textit{Vittorio Gassman, Vittorio De Sica, Sophia Loren}?
\item[User:] I like Vittorio Gassman.
\item[System:] Based on your answers, I recommend \textit{Big Deal on Madonna Street}. What do you think?
\item[User:] Great, I'll try this one!
\end{description}
\noindent \textup{The recommendation is successful. The dialog is over.}
\qed
\end{example}

\noindent\begin{example}[Protocol 2 - The user rejects a recommendation and the CRS asks to declare a disliked feature.]
~\newline

\begin{description}
\item[System:] What kind of genre do you prefer among \textit{Western, Historical} and \textit{Love Story}?
\item[User:] I would like to explore Western movies.
\item[System:] Who is one of your favorite film score composers? I have these options: \textit{Ennio Morricone}, \textit{Alessandro Alessandroni}, \textit{Paul Sawtell}, \textit{Frank Skinner}.
\item[User:] I like Ennio Morricone.
\item[System:] Who is one of the best actor among \textit{Leonardo Di Caprio}, \textit{Clint Eastwood}, \textit{Gian Maria Volonté} and \textit{Tim Roth} in your opinion?
\item[User:] I like Leonardo Di Caprio.
\item[System:] Do you like \textit{Quentin Tarantino} as film director?
\item[User:] I like Leonardo Di Caprio.
\item[System:] Based on your answers  I recommend \textit{Django Unchained}. Do you think you'll like it?
\item[User:] No, I don't.
\item[System:] With the choices you gave me, I have no more matching movies. Would you mind reconsidering one of your stated  preferences?
\item[User:] Ok, let's forget about Leonardo Di Caprio. I also like Tim Roth.
\end{description}
\noindent \textup{At this point the system discards movies by \textit{Leonardo Di Caprio} and concentrates on movies starring \textit{Tim Roth}.}
\begin{description}
\item[System:] Do you like \textit{Robert Richardson} as director of photography?
\item[User:] Sure, I know Robert Richardson as a good director of photography.
\item[System:]  Would you prefer movies that have won an Oscar?
\item[User:] Yes, I do.
\item[System:] Based on your answers, I recommend \textit{The Hateful Eight}. Will you try it out?
\item[User:] Sounds good, I will.
\end{description}
\noindent \textup{The recommendation is successful. The dialog is over.}
\qed
\end{example}

\subsection{Results and Discussion} \label{sec:discussion}
In this section we show the experimental results. For each itemset (namely, \textit{IS1} and \textit{IS2}) we compared the two protocols (\ie \textit{P1} and \textit{P2}) and we counted the Number of Questions (NQ) required for reaching the test item $\hat{I}$.
Figures~\ref{plot:avg_questions} and~\ref{plot:max_questions} report the average number of questions required to reach a test item (the average is computed on all the configurations), and the overall maximum number of questions for each protocol and itemset.
\pgfplotsset{compat=1.16}
\begin{figure}[H]
\footnotesize
\centering
    \begin{subfigure}[b]{0.45\textwidth}
    \centering
        \begin{tikzpicture}
            \begin{axis}[
            every axis plot post/.style={/pgf/number format/fixed},
            width=1.2\textwidth,
            ybar=10pt,
            bar width=12pt,
            ymin=0,
            axis on top,
            ymax=200,
            xtick=data,
            axis y line=none, 
            axis x line=bottom,
            enlarge x limits=0.5,
            every node near coord/.append style={font=\scriptsize},
            legend style={at={(0,1.3)},nodes={scale=0.8, transform shape},anchor=north west},
            symbolic x coords={Itemset1, Itemset2},
            restrict y to domain*=0:220, 
            visualization depends on=rawy\as\rawy, 
            after end axis/.code={ 
                    \draw [ultra thick, white, decoration={snake, amplitude=1pt}, decorate] (rel axis cs:0,1.05) -- (rel axis cs:1,1.05);
                },
            nodes near coords={%
                    \pgfmathprintnumber{\rawy}
                },
            axis lines*=left,
            clip=false
            ]
            \addplot coordinates {(Itemset1, 72.64025945) (Itemset2, 1023.745268)}; 
            \addplot coordinates {(Itemset1, 67.45143592) (Itemset2, 166.344615)};
            \legend{Protocol 1, Protocol 2}; 
            \end{axis} 
        \end{tikzpicture}
    \caption{Average number of Questions per Configuration}
    \label{plot:avg_questions}
    \end{subfigure}
\hfill
\begin{subfigure}[b]{0.45\textwidth}
\centering
\begin{tikzpicture}
\begin{axis}[
    every axis plot post/.style={/pgf/number format/fixed},
    width=1.2\textwidth,
    ybar=10pt,
    bar width=12pt,
    ymin=0,
    axis on top,
    ymax=2000,
    xtick=data,
    axis y line=none, 
    axis x line=bottom,
    enlarge x limits=0.5,
    every node near coord/.append style={font=\scriptsize},
    legend style={at={(0,1.3)},nodes={scale=0.8, transform shape},anchor=north west},
    symbolic x coords={Itemset1, Itemset2},
    restrict y to domain*=0:2500, 
    visualization depends on=rawy\as\rawy, 
    after end axis/.code={ 
            \draw [ultra thick, white, decoration={snake, amplitude=1pt}, decorate] (rel axis cs:0,1.05) -- (rel axis cs:1,1.05);
        },
    nodes near coords={%
            \pgfmathprintnumber{\rawy}
        },
    axis lines*=left,
    clip=false
    ]
\addplot coordinates {(Itemset1, 1588) (Itemset2, 6671)}; 
\addplot coordinates {(Itemset1, 1152) (Itemset2, 524)};
\legend{Protocol 1, Protocol 2}; 
\end{axis} 
\end{tikzpicture}
\caption{Maximum number of Questions in Configurations}
\label{plot:max_questions}
\end{subfigure}
\caption{}
\end{figure}
As expected from the theoretical analysis, with IS1 we observe minor differences
between the two protocols in terms of number of questions made by the system before reaching and recommending the test item. More specifically, P1 needs 72.64 as average number of questions with IS1, whereas P2 needs 67.45 questions. The difference is however huge for IS2 where P1 needs more than 1,000 questions on average to reach the test item, while P2 needs around 166 questions. 
Hence, we can confirm that when the items in the catalog have several features with few distinct values, the efficiency of P2 grows drastically compared to P1. Also \emph{maximum} number of questions---reported in Figure~\ref{plot:max_questions} for each pair of protocol and itemset---confirms this different efficiency for IS1 and IS2. 

We now discuss the fact that NQ is very large---indeed, also in the best combination (P2 with IS1), the number of questions is close to 70, which sounds unrealistic for any real user. It is worth noting that in this experiment we implemented the very worst-case scenario. Our experiment used an unrealistic setting on purpose. In our particular scenario the recommendation task is deliberately very difficult:
\begin{itemize}
\item for each dialog, there is only one test item (true positive) to be found in the catalog and we do not consider the usual possibility that the user would accept several items;
\item the CRS works in cold-start condition without any user profile, thus several questions are needed for acquiring the user preferences;
\item the CRS does not implement a cut-off on the number of questions to ask the user.
\end{itemize}



In conclusion, the results of our experimental evaluation are aligned with our theoretical findings and we can thus confirm both the research hypotheses H1 and H2. 

In other words, our simulation confirmed what was foreseen by the theoretical analysis, namely that the difference between protocol~P1 and protocol~P2 shows up clearly only in a dataset in which there are many features  with a small set of different values.

\subsection{Limitations}
Even though our preliminary experimental evaluation confirms our theoretical result, we identified some limitations that deserve a further investigation. The first limitation of this study is related to the simulation of the dialog. The average number of questions as well as the maximum number of questions, also in the best scenario, is too large for a real-world context. This is due to the fact that we simulated the very worst-case scenario.
Furthermore, a study where the complexity analysis is done with different catalog configurations, from the worst to the best, could help in a better understanding of the final results.
Moreover, a user study in a real scenario where the user can choose more than one recommendation, and the recommendation specifically reflects her interests could surely be very helpful. 
Another limitation of this work is the domain. There are domains where the items in the catalog are characterized by few features with a very large number of possible values (as Itemset2 in our simulation). Currently, we are not able to propose a strategy for dealing with these situations.


\section{Summary and Outlook}\label{sec:conclusion}
While conversational approaches to recommendation have been explored since the late 1990s, we observe an increasing interest in these types of system in recent years in particular due to recent developments in natural language processing. Independent of the interaction modality, efficiency in terms of required dialog turns has commonly been a main target of research. So far, however, questions of efficiency were almost exclusively investigated in an empirical manner. With this work, we contribute to a better understanding of theoretical properties of conversational recommendation problems and we specifically address questions related to the computational complexity of finding efficient dialog strategies. One main insight of our theoretical analysis---which was also confirmed by a computational experiment---is that when designing an efficient conversation strategy, we must always consider the characteristics of the item catalog. 

More specifically, we demonstrated that when the items in the catalog are characterized by a few features with a large number of distinct values, the critiquing strategy based on asking the user about a disliked characteristic of the recommended item does not give any significant advantage in terms of user effort. Conversely, when the catalog is composed of items with several features with a few distinct values, a critique strategy based on item features can drastically reduce the user effort for reaching a liked recommendation.

On a more general level, we hope that our work might help to stimulate more theory-oriented research in this area, leading us to a better understanding of the foundational properties of this essential class of interactive AI-based systems.


The problem formalization chosen in our work targets the predominant class of \crs in the literature---systems that  interactively elicit user preferences regarding item features---and covers a set of common interaction types in these types of systems. For the purpose of the current study, we however limited our formalization to two basic preference elicitation steps: \emph{(a)} by providing information about individual feature values, and \emph{(b)} by providing feedback for individual items as a whole. In our future work, we plan to extend the formalization to allow users to state preferences on combination of feature values, \eg ``\textit{I don't like open-air Japanese restaurants}'', or ``\textit{I like ocean-view, romantic Italian restaurants}''. Other simplifications in our current work that will be addressed in the future include the consideration of multi-valued features, both in the queries and the item catalog, and situations where individual item features might be unknown.

Finally, we see another and potentially more far-reaching direction for future research in the explicit consideration of individual long-term user preferences in the interactive recommendation process.

\bibliographystyle{elsarticle-num}

\bibliography{kr-and-complexity, recommenderSystems}


\end{document}